\documentclass[a4paper,10pt]{article}
\usepackage[T1]{fontenc}
\usepackage{parskip}
\usepackage[font={small}]{caption}
\usepackage[hidelinks]{hyperref}
\usepackage[margin=2.5cm]{geometry}
\usepackage{times}
\usepackage{amsmath,amssymb,amsthm}
\usepackage[affil-it]{authblk}
\usepackage{listings}
\usepackage[square,numbers]{natbib}

\usepackage{amsmath,amsfonts,amssymb,bm} % For nice maths
\usepackage{array,booktabs} 			% For nice tables
\usepackage{graphicx}
\usepackage{hyperref}
\usepackage{color}
\usepackage{soul}
\usepackage{float}
\usepackage{placeins}
\usepackage{siunitx}

\usepackage{algorithm}
\usepackage{algorithmic}

\DeclareMathOperator{\trace}{tr}

\newcommand{\Transp}{\mathsf{T}}
\newcommand{\pn}{v} % process noise
\newcommand{\mn}{e} % measurement noise

\newtheorem{lemma}{Lemma}
\newtheorem{assumption}{Assumption}

\hypersetup{
	colorlinks=true,       % false: boxed links; true: colored links
	linkcolor=blue,          % color of internal links
	citecolor=blue,        % color of links to bibliography
	urlcolor=blue           % color of external links
}

\date{\today}

\title{Gaussian Variational State Estimation for Nonlinear State-Space Models\thanks{This work has been submitted to the IEEE for possible publication. Copyright may be transferred without notice, after which this version may no longer be accessible.}}
\author[1]{Jarrad Courts\thanks{\url{Jarrad.Courts@uon.edu.au}}}
\author[2]{Adrian Wills\thanks{\url{Adrian.Wills@newcastle.edu.au}}}
\author[3]{Thomas B. Sch\"on\thanks{\url{thomas.schon@it.uu.se}}, \thanks{This research was financially supported by \emph{Kjell och M{\"a}rta Beijer Foundation}, and via the projects \emph{Learning flexible models for nonlinear dynamics} (contract number: 2017-03807), and \emph{NewLEADS - New Directions in Learning Dynamical Systems} (contract number: 621-2016-06079), both funded by the Swedish Research Council.}}

\affil[1]{School of Engineering, University of Newcastle, Australia}
\affil[2]{School of Engineering, University of Newcastle, Australia}
\affil[3]{Department of Information Technology, Uppsala University, Sweden}

\begin{document}

% make the title area
\maketitle

% !TeX root = VI_state_estimation_tsp.tex

\begin{abstract}
  In this paper, the problem of state estimation, in the context of both filtering and smoothing, for nonlinear state-space models is considered. Due to the nonlinear nature of the models, the state estimation problem is generally intractable as it involves integrals of general nonlinear functions and the filtered and smoothed state distributions lack closed-form solutions. As such, it is common to approximate the state estimation problem. In this paper, we develop an assumed Gaussian solution based on variational inference, which offers the key advantage of a flexible, but principled, mechanism for approximating the required distributions. Our main contribution lies in a new formulation of the state estimation problem as an optimisation problem, which can then be solved using standard optimisation routines that employ exact first- and second-order derivatives. The resulting state estimation approach involves a minimal number of assumptions and applies directly to nonlinear systems with both Gaussian and non-Gaussian probabilistic models. The performance of our approach is demonstrated on several examples; a challenging scalar system, a model of a simple robotic system, and a target tracking problem using a von Mises-Fisher distribution and outperforms alternative assumed Gaussian approaches to state estimation.  
\end{abstract}

% !TeX root = VI_state_estimation_tsp.tex

\section{Introduction}
	The problem of state estimation is of great interest to a wide variety of practical scientific and engineering problems \cite{Jazwinski1970,Saerkkae2013}. The essential problem arises in situations where state-space models have been developed to describe the dynamic behaviour of a system and the full state is not directly measured. Instead, other related measurements are available, which are called the system outputs. In this situation, it is important to describe both the time evolution of the state (the state transition model), and the relationship between the state and available measurements (the measurement model). 

	In this paper, these two models are represented as conditional probability distributions given by
	\begin{subequations} \label{SE:eq:general nonlinear model}
		\begin{align}
			x_{k+1} &\sim p_\theta\left(x_{k+1} \mid x_k \right), \\
			y_k &\sim p_\theta\left(y_k \mid x_k\right),
		\end{align}
	\end{subequations}
	respectively. In the above, \(\theta \in \mathcal{R}^{n_\theta}\) is a vector of model parameters, \( x_k \in \mathcal{R}^{n_x} \) is the state at time \(k\), and \(y_k \in \mathcal{R}^{n_y}\) is the observed measurement at time \(k\). Throughout this paper, it is assumed that a value for the parameter \(\theta\) is provided. Whilst we do consider the presence of an input \(u_k \in \mathcal{R}^{n_u}\) throughout this paper, it does not influence the approach; as such, the notation has been suppressed for clarity. 

	Given a sequence of \(T >0\) measurements denoted as \(y_{1:T} = \{y_1, \dots. y_T\}\) and a prior distribution \(p\left(x_0\right)\) on the initial state \( x_0 \), we are concerned with both the filtering problem of obtaining \(p_\theta\left(x_k \mid y_{1:k}\right)\), and the smoothing problem of obtaining each pairwise smoothed distribution, \(p_\theta\left(x_{k-1},x_{k} \mid y_{1:T}\right)\) for \(k \in 1, \dots, T\). 

	In principle, both of these state estimation problems can be solved recursively with the use of Bayes' rule and the law of total probability to describe each state as a probability density function \cite{Jazwinski1970}. For example, the filtering distribution is provided by the following recursions
	\begin{subequations}
		\label{eq:filtEqns}
		\begin{align}
			p_\theta(x_{k} \mid y_{1:k}) &= \frac{p_\theta(y_k \mid x_k)\, p_\theta(x_k \mid y_{1:k-1})}{p_\theta(y_k \mid y_{1:k-1})},\label{eq:1}\\
			p_\theta(x_{k+1} \mid y_{1:k}) &= \int p_\theta(x_{k+1} \mid x_k) \, p_\theta(x_k \mid y_{1:k}) \textnormal{d} x_k.\label{eq:2}
		\end{align}
	\end{subequations}
	There are two essential problems with computing this recursion. The first is that both \eqref{eq:1} and \eqref{eq:2} have no known closed-form solutions for general nonlinear models in the form of \eqref{SE:eq:general nonlinear model}. The second is that \eqref{eq:2} involves multi-dimensional integration, which even for modest state dimensions will be intractable to solve with high accuracy \cite{Gordon1993}. The state smoothing problem recursions face the same difficulties.

	These challenges have motivated a large body of research directed towards approximating the filtered and smoothed distributions, and the pertinent research is outlined in Section~\ref{SE:sec:related works}.

	\textbf{Contributions:} Relative to the existing work, the main contribution of this paper is to present a variational inference (VI) \cite{Jordan1999,Blei2017} approach to state estimation for nonlinear discrete-time state-space models. In general, the VI approach aims to approximate posterior distributions with a parametric density of an assumed form; in this paper, the densities are assumed Gaussian. Within the current context, this can be employed to approximate the desired state filtering and smoothing posteriors. The key lies in carefully tayloring the parametric form of the assumed posterior distributions, not relying on the classic mean-field approximation.

	The primary mechanism of VI is to minimise the Kullback-Leiber (KL) divergence \cite{Kullback1951} between the assumed posterior distribution and the intractable posterior. Hence, the inference problem is cast as an optimisation problem \cite{Blei2017}.

	Differing from existing work, the presented approach relies upon standard optimisation routines that employ exact gradients and exact Hessians to address both the filtering and smoothing problems; due to the selected parametrisation, this can be efficiently performed. Additionally, the state estimation approach uses a minimal number of assumptions and is directly applicable to both Gaussian and non-Gaussian nonlinear systems. The performance and robustness of the proposed method is demonstrated on several examples. We have partially presented the approach detailed in the unpublished work of \cite{Courts2020}.

	The remainder of this paper is organised as follows; Section~\ref{SE:sec:related works} provides a brief overview of related work followed by Section~\ref{SE:sec:main part} detailing the proposed state estimation approach for both filtering and smoothing. Section~\ref{SE:sec:implmentation} then provides implementation-specific details followed by numerical examples in Section~\ref{SE:sec:examples}. Section~\ref{SE:sec:conculsion} concludes the paper.

% !TeX root = VI_state_estimation_tsp.tex

\section{Related Work}\label{SE:sec:related works}
	Particle filters \citep{Gordon1993} and particle smoothers \citep{Doucet2000} can directly handle both the filtering and smoothing problems for nonlinear models. Fundamentally, these methods rely on the Law of Large Numbers for approximating integrals with finite sums. This approximation converges to the correct solution aymptotically as the number of samples tends to infinity. As such, these particle methods can be computationally expensive, even for modest state dimensions. For this reason, alternative approaches have been widely explored and the majority of these alternatives rely on an assumed posterior density, either impicitly or explicitly. 

	Amongst the most popular assumed density approaches are the extended Kalman filter (EKF) \citep{Jazwinski1970}, the unscented Kalman filter (UKF) \citep{Julier1997}, and the respective smoothers, the extended Kalman smoother (EKS) \citep{Cox1964}, and the unscented Rauch-Tung-Striebel smoother (URTSS) \citep{Saerkkae2008}. These approaches assume a Gaussian density and are based on the Kalman filter \citep{Kalman1961} and Rauch-Tung-Striebel (RTS) \citep{Rauch1965} smoother, which provide exact solutions for linear Gaussian systems.

	These approaches function by approximating the nonlinear models about the state prior \citep{GarciaFernandez2015}. Iteratively recomputing this approximation using the posterior distribution, however, can improve the performance \citep{GarciaFernandez2015}. Many iterated EKF and UKF based approaches exist; the iterated posterior linearization filter (IPLF) \citep{GarciaFernandez2015} and the iterated UKF approaches in \citep{Skoglund2019} are two examples. The iterated posterior linearization smoother (IPLS) \citep{GarciaFernandez2017}  extends this approach to the smoothing problem and is generalised in \citep{Tronarp2018} to allow non-Gaussian models. 

	All these approaches, however, are based upon a Kalman filtering framework \citep{GarciaFernandez2015}. This limits the form of the state corrections and, for nonlinear models, can be outperformed using alternative approaches \citep{Darling2017}. The minimum divergence filter (MDF) \citep{Darling2017} is one such approach derived from a KL divergence perspective. The MDF, however, is noniterative and can perform poorly with small measurement noise, which can require the artificial introduction of noise \citep{Darling2017}. The stochastic search Kalman filter (SKF) \citep{Gultekin2017} is another assumed Gaussian filter based on minimising a KL divergence, in this case, by using a Monte Carlo search. The SKF, however, was outperformed by moment matching.

	This KL divergence basis is also known as variational inference. VI is also used in \citep{Wang2019} to provide a variational iterated filter (VIF) and in \citep{Vrettas2011,Vrettas2015,Vrettas2008,AlaLuhtala2015} to smooth stochastic differential equations. More recent work includes exactly sparse Gaussian variational inference (ESGVI) \citep{Barfoot2020}, a batch state estimation method. Intended for nonlinear systems, ESGVI focuses on simultaneous localisation and mapping (SLAM), derivative-free methods, and the derivation of an optimiser based on approximations to first- and, some, second-order derivatives.

	A key difference between these related works and the approach proposed in this paper is \textit{where} and \textit{when} approximations are handled and introduced. In this paper, only one assumption (selected density form) and one approximation (Gaussian quadrature of integrals) are introduced, and only when required to produce tractable optimisation problems. The resulting optimisation problems are of a standard form and they can directly solved without further simplifications using exact first- and second-order derivatives to a local maximum. Due to the careful parametrisation of the problem, these derivatives are readily available. 

	Contrary to the developed approach, many related works, particularly the Kalman based approaches, implicitly introduce many inherited assumptions not necessarily appropriate or justified for nonlinear systems. A further contrast is when and how the iterations and optimisation are performed for both the Kalman and VI-based approaches. Generally speaking, the approaches are `ad hoc', incorporating some ideas from Newton-style optimisation; but they do not necessarily form efficient and robust procedures such as those detailed in the optimisation literature, for example, \cite{Nocedal2006}. As such, convergence difficulties or even divergence may occur, a problem observed in many of the iterated approaches.

	Furthermore, these alternative approaches typically introduce additional simplifications or approximations when addressing the resulting optimisation problems. This often occurs in the calculation of derivatives, where approximations to the gradient are frequent. Using gradient approximations limits the range of optimisation routines appropriate to use as many assume exact gradients \citep{Nocedal2006}. 

	Finally, the use of second-order derivatives is infrequent addressed by the existing approaches. When considered, second-order derivatives are typically approximated, and only for the components relating to the mean, and not the covariance, of the assumed distribution. While standard optimisation routines accommodate using approximate Hessians, this represents a difference from the approach we develop, which allows for the exact Hessian to be efficiently obtained and used in a way that allows for the use of standard optimisation routines.

% !TeX root = VI_state_estimation_tsp.tex

\section{Variational Nonlinear State Estimation} \label{SE:sec:main part}
	In this section, we first develop the proposed approach to state estimation in the context of smoothing in Section~\ref{SE:sec:smoothing}. The filtering problem is then examined in Section~\ref{SE:sec:filtering}, which is followed by an illustrative example and discussion in Section~\ref{SE:sec:discussion}.

	% !TeX root = VI_state_estimation_tsp.tex

\subsection{Smoothing} \label{SE:sec:smoothing}
	As \( p_\theta(x_{0:T} \mid y_{1:T}) \) cannot be represented in closed-form it will be approximated using an assumed density, parameterised by \(\beta\), and denoted as \(q_\beta(x_{0:T})\). At a high level, assumed density smoothing consist of two steps; firstly, selecting a tractable parametric form, and secondly, for a given parametric form, obtaining a numeric value for \(\beta\). For a given parametric form, approaches to assumed density smoothing differ by how the value of \(\beta\) is obtained. In this paper, the value of \(\beta\) is obtained using variational inference.
	
	Our approach to obtaining values for \(\beta\) begins by utilising conditional probability to express the log-likelihood, \(\log p_\theta(y_{1:T})\), as
	\begin{align} \label{SE:eq:log_cond_prob}
		\log p_\theta(y_{1:T}) &= \log p_\theta(x_{0:T}, y_{1:T}) - \log p_\theta(x_{0:T} \mid y_{1:T} ),
	\end{align}
	independent of \(q_\beta(x_{0:T})\). The assumed density is introduced by adding and subtracting \( \log q_\beta(x_{0:T}) \) to the right-hand side of \eqref{SE:eq:log_cond_prob}. With minor rearrangement this leads to 
	\begin{align} \label{SE:eq:LL_sum}
		\log p_\theta(y_{1:T}) &= \log \frac{ p_\theta(x_{0:T}, y_{1:T})  }{ q_\beta(x_{0:T}) } + \log \frac{ q_\beta(x_{0:T}) }{ p_\theta(x_{0:T} \mid y_{1:T} )}.
	\end{align}
	As \(\log p_\theta(y_{1:T}) \) is independent of \(x_{0:T}\), we can write
	\begin{align} \label{SE:eq:LL_independence}
		\log p_\theta(y_{1:T}) = \int  q_\beta(x_{0:T}) \log p_\theta(y_{1:T}) dx_{0:T}.
	\end{align}
	By substitution of \eqref{SE:eq:LL_sum} into the right-hand side of \eqref{SE:eq:LL_independence}, we arrive at
	\begin{align}
		\log p_\theta(y_{1:T}) = &\int q_\beta(x_{0:T}) \log \frac{ p_\theta(x_{0:T}, y_{1:T})  }{ q_\beta(x_{0:T}) } dx_{0:T} \notag \\
								 &\quad+ \int  q_\beta(x_{0:T}) \log \frac{ q_\beta(x_{0:T}) }{ p_\theta(x_{0:T} \mid y_{1:T} )} dx_{0:T}, 
	\end{align}
	which is equivalent to
	\begin{align} \label{SE:eq:ll_equal_L_plus_KL}
		\log p_\theta(y_{1:T}) = \mathcal{L}\left(\beta\right) + \text{KL}[ q_\beta(x_{0:T})  \mid\mid p_\theta(x_{0:T} \mid y_{1:T} )],
	\end{align}
	where \( \text{KL}[ q_\beta(x_{0:T})  \mid\mid p_\theta(x_{0:T} \mid y_{1:T} )] \) is the KL divergence of \( p_\theta(x_{0:T} \mid y_{1:T} )\) from \( q_\beta(x_{0:T}) \) and
	\begin{align}
		\mathcal{L}\left(\beta\right) = \int q_\beta(x_{0:T}) \log \frac{ p_\theta(x_{0:T}, y_{1:T})  }{ q_\beta(x_{0:T}) } dx_{0:T}.
	\end{align}
	As \( \text{KL}[ q_\beta(x_{0:T})  \mid\mid p_\theta(x_{0:T} \mid y_{1:T} )] \geq 0\), from \eqref{SE:eq:ll_equal_L_plus_KL}, it is seen that
	\begin{align}
		\log p_\theta(y_{1:T}) \geq \mathcal{L}\left(\beta\right),
	\end{align}
	and hence \(\mathcal{L}\left(\beta\right) \) is a lower bound to the log-likelihood, \(\log p_\theta(y_{1:T})\).
	
	The proposed approach to smoothing consists of selecting a value for \(\beta\), denoted \(\beta^\star\), by maximising this lower bound to the log-likelihood via
	\begin{align} \label{SE:eq:highest_level_optim}
		\beta^\star &= \arg\max_\beta \quad \mathcal{L}(\beta).
	\end{align}
	The resulting smoothed density is then given by \(  q_{\beta^\star}(x_{0:T}) \).
	
	The use of \eqref{SE:eq:highest_level_optim} to obtain numerical values of \(\beta\) poses several appealing properties. Firstly, no representation of  \(p_\theta(x_{0:T} \mid y_{1:T}) \) is required; this avoids one of the problems associated with the direct use of Bayes' rule to perform smoothing. 
	
	Secondly, \emph{if} the parametric form of \(q_\beta(x_{0:T})\) matches \(p_\theta(x_{0:T} \mid y_{1:T}) \) the KL divergence term can be driven to zero by the maximisation. This means \(  q_{\beta^\star}(x_{0:T}) \) is identical to \(p_\theta(x_{0:T} \mid y_{1:T}) \) and smoothing has been performed exactly. 
	
	Thirdly, from \eqref{SE:eq:ll_equal_L_plus_KL} and due to both the nonnegativity of any KL divergence and \(\log p_\theta(y_{1:T}) \) being constant, maximisation of \(\mathcal{L}\left(\beta\right)\) is equivalent to minimisation of \(\text{KL}[ q_\beta(x_{0:T})  \mid\mid p_\theta(x_{0:T} \mid y_{1:T} )]\). This is the motivation behind maximising \(\mathcal{L}\left(\beta\right)\), it allows the KL divergence of \( p_\theta(x_{0:T} \mid y_{1:T} )\) from \( q_\beta(x_{0:T}) \) to be minimised without requiring \( p_\theta(x_{0:T} \mid y_{1:T} )\).	
	
	Despite these appealing properties performing estimation in this way, i.e. minimising \(\text{KL}[ q_\beta(x_{0:T})  \mid\mid p_\theta(x_{0:T} \mid y_{1:T} )]\), is not without drawbacks. For example, it is well-known that when \(q_\beta(x_{0:T})\) cannot adequately represent \(p_\theta(x_{0:T} \mid y_{1:T} )\), mode seeking behaviour and over-certainty can occur \citep{Bishop2006,Barber2012}. Among other situations, this behaviour can occur when \(q_\beta(x_{0:T})\) is selected to belong to an uni-modal family while \(p_\theta(x_{0:T} \mid y_{1:T} )\) is multi-modal and can be a limitation of the approach.	
	
	This overall approach of approximating an intractable inference problem with an optimisation problem is known as variational inference. The VI approach has been utilised in many fields, and a particularly prominent example is the field of machine learning \cite{Zhang2019}.
	
	While \eqref{SE:eq:highest_level_optim} avoids \(p_\theta(x_{0:T} \mid y_{1:T}) \) it remains generally intractable. One reason is that \( q_\beta(x_{0:T}) \) is an assumed density to the full posterior distribution; with an increasing number of time steps, calculations involving the full posterior distribution \( p_\theta(x_{0:T} \mid y_{1:T} )\), or approximations thereof,  will become computationally prohibitive \citep{Saerkkae2013}. However, the smoothing problem this paper considers is obtaining assumed density approximations \( q_\beta(x_{k-1},x_{k}) \), to each pairwise joint smoothed distribution \(p_\theta\left(x_{k-1},x_{k} \mid y_{1:T}\right)\), and not to obtain \( q_\beta(x_{0:T}) \).
	
	To utilise the approach given by \eqref{SE:eq:highest_level_optim} it is required to express \( \mathcal{L}\left(\beta\right)\) as a function of each \( q_\beta(x_{k-1},x_{k}) \) rather than  \( q_\beta(x_{0:T}) \). To achieve this we begin by separating \( \mathcal{L}\left(\beta\right)\) as
	\begin{subequations}
		\begin{align}
			\mathcal{L}\left(\beta\right) &= \int q_\beta(x_{0:T}) \log \frac{ p_\theta(x_{0:T}, y_{1:T})  }{ q_\beta(x_{0:T}) } dx_{0:T} \\
										  &= \int q_\beta(x_{0:T}) \log p_\theta(x_{0:T}, y_{1:T})  dx_{0:T} \notag \\
										  &\quad - \int q_\beta(x_{0:T}) \log q_\beta(x_{0:T}) dx_{0:T}. \label{SE:eq:full_L}
		\end{align}
	\end{subequations}
	From conditional probability, and the conditional independence of the measurements, we can write
	\begin{align}
		p_\theta(x_{0:T}, y_{1:T}) &= p(x_0) \prod_{k=1}^{T} p_\theta(x_{k}, y_{k} \mid x_{k-1}).
	\end{align}
	This allows \(\log p_\theta(x_{0:T}, y_{1:T})\) to be written as
	\begin{align*}
		\log p_\theta(x_{0:T}, y_{1:T}) &= \log \left(  p(x_0) \prod_{k=1}^{T} p_\theta(x_{k}, y_{k} \mid x_{k-1}) \right) \\
		&= \log p(x_0) + \log \prod_{k=1}^{T} p_\theta(x_{k}, y_{k} \mid x_{k-1})          \\
		&= \log p(x_0) + \sum_{k=1}^{T} \log p_\theta(x_{k}, y_{k} \mid x_{k-1}),
	\end{align*}
	and hence the first integral in \eqref{SE:eq:full_L} can be expressed as
	\begin{align}
		\int q_\beta(x_{0:T}) \log p_\theta(x_{0:T}, y_{1:T})  dx_{0:T} = I_1\left(\beta\right) + I_{23}\left(\beta\right),
	\end{align}
	where
	\begin{subequations}
		\begin{align}
			I_1\left(\beta\right) 		&= \int q_\beta(x_{0:T}) \log p(x_0)  dx_{0:T}, \\
			I_{23}\left(\beta\right) 	&= \int q_\beta(x_{0:T})  \sum_{k=1}^{T} \log p_\theta(x_{k}, y_{k} \mid x_{k-1})  dx_{0:T}  \\
										&= \sum_{k=1}^{T} \int q_\beta(x_{0:T})  \log p_\theta(x_{k}, y_{k} \mid x_{k-1})  dx_{0:T}  \\
										&= I_2\left(\beta\right) + I_3\left(\beta\right),
		\end{align}
		and
		\begin{align}
			I_2\left(\beta\right)  &= \sum_{k=1}^{T} \int q_\beta(x_{0:T}) \log p_\theta(x_{k} \mid x_{k-1})  dx_{0:T}, \\
			I_3\left(\beta\right)  &= \sum_{k=1}^{T} \int q_\beta(x_{0:T}) \log p_\theta(y_{k} \mid x_{k})  dx_{0:T}.
		\end{align}
	\end{subequations}
	Due to the fact that \( p_\theta(x_{k} \mid x_{k-1}) \) and \( p_\theta(y_{k} \mid x_{k}) \) only depend upon a small subset of \(x_{0:T}\) and since \(q_\beta(x_{0:T})\) is a probability distribution, the dimension of these integrals can be reduced to
	\begin{subequations}
		\begin{align}
			I_1\left(\beta\right)  &= \int q_\beta(x_0) \log p(x_0)  dx_0, \\
			I_2\left(\beta\right)  &= \sum_{k=1}^{T} \int q_\beta(x_{k-1:k}) \log p_\theta(x_{k} \mid x_{k-1})  dx_{k-1:k}, \\
			I_3\left(\beta\right)  &= \sum_{k=1}^{T} \int q_\beta(x_{k}) \log p_\theta(y_{k} \mid x_{k})  dx_{k}.
		\end{align}
	\end{subequations}

	To address the second integral in \eqref{SE:eq:full_L} in a pairwise fashion the following assumption is introduced.
	\begin{assumption} \label{SE:ass:suffienct flexibity}
		The assumed density \(q_\beta\left(x_{0:T}\right)\) is selected such that \(\beta\) can be divided into \(\{\beta_1,\beta_2\}\), where \(\beta_2\) does not influence each pairwise marginal distribution of \(q_\beta\left(x_{0:T}\right)\), i.e. \(q_\beta\left(x_{k-1},x_k\right) = q_{\beta_1}\left(x_{k-1},x_k\right)\), and that, for a given \(\beta_1\), a \(\beta_2\) exists such that
		\begin{align}
			q_{\beta_1,\beta_2}(x_{k} \mid x_{k-1}, x_{0:k-2}) =  q_{\beta_1,\beta_2}(x_{k} \mid x_{k-1}).
		\end{align}
	\end{assumption}
	\begin{lemma}  \label{SE:lem:optimal_q_factorises}
		Under Assumption~\ref{SE:ass:suffienct flexibity}, the optimal full assumed density \(q_{\beta^\star}\left(x_{0:T}\right)\)	factorises according to	
		\begin{align} \label{SE:eq:conditional_factorisation}
			q_{\beta^\star}\left(x_{0:T}\right) = q_{\beta^\star}(x_0) \prod_{k=1}^{T} q_{\beta^\star}(x_{k} \mid x_{k-1}).
		\end{align}
	\end{lemma}
	\begin{proof}
		See Appendix~\ref{SE:app:proof of factorisation lemma}.
	\end{proof}
	\begin{lemma} \label{SE:lem:entropy_decomposition}
		For the purpose of finding each optimal pairwise distribution \(q_{\beta^\star}\left(x_{k-1},x_k\right)\), the second integral in \eqref{SE:eq:full_L} can be calculated as
		\begin{align}
			\int q_\beta(x_{0:T}) \log q_\beta(x_{0:T}) dx_{0:T} &= I_4\left(\beta\right),
		\end{align}
		where	
		\begin{align}
			I_4\left(\beta\right) &= \sum_{k=1}^{T}\int q_\beta(x_{k-1:k}) \log q_\beta(x_{k-1:k}) dx_{k-1:k} \notag \\
			&\quad - \sum_{k=1}^{T-1} \int q_\beta(x_{k}) \log q_\beta(x_{k}) dx_{k}, \label{SE:eq:expanded_I4}
		\end{align}
		provided \(q_\beta\left(x_{0:T}\right)\) satisfies Assumption~\ref{SE:ass:suffienct flexibity}.
	\end{lemma}
	\begin{proof}
		See Appendix~\ref{SE:app:entropy_decomposition}.
	\end{proof}

	This simplifies the calculation of the lower bound to 
	\begin{subequations}
		\begin{align}
			\mathcal{L}\left(\beta\right) &= \int q_\beta(x_{0:T}) \log \frac{ p_\theta(x_{0:T}, y_{1:T})  }{ q_\beta(x_{0:T}) } dx_{0:T}  \\
			&= I_1\left(\beta\right) + I_2\left(\beta\right) + I_3\left(\beta\right) - I_4\left(\beta\right).
		\end{align}
	\end{subequations}
	From \(I_1\left(\beta\right)\), \(I_{2}\left(\beta\right)\), \(I_3\left(\beta\right)\), and \(I_4\left(\beta\right)\) it is seen that the calculation of \(\mathcal{L}(\beta)\) never requires the full distribution \( q_\beta(x_{0:T})\), instead it is just the pairwise joint  distributions \( q_\beta(x_{k-1},x_{k}) \) and marginals thereof that are required.

	It remains to detail how to address the generally intractable integrals in \(I_1\left(\beta\right)\), \(I_{2}\left(\beta\right)\), \(I_3\left(\beta\right)\), and \(I_4\left(\beta\right)\) and how to perform the maximisation. While critical from an implementation point of view, these details are conceptually unimportant. As such, they are temporarily deferred in favour of examining the application of this approach to a linear system and the minor change required to apply this approach to more general state-space models.

	\subsubsection{Application to linear and nonlinear systems}
		Consider a linear state-space model with additive Gaussian noise according to
		\begin{subequations}
			\begin{align}
				x_{k+1} &= A x_k + B u_k + \pn_k, \\
				y_k &= C x_k + D u_k + \mn_k,
			\end{align}
			where
			\begin{align}
				\pn_k \sim \mathcal{N}\left(0,Q\right), \qquad \mn_k \sim \mathcal{N}\left(0,R\right),
			\end{align}
		\end{subequations}
		\(A\), \(B\), \(C\), \(D\), \(Q\), and \(R\) are model parameters, \(u_k\) is a known input, and the initial state prior is given by \( p(x_0) = \mathcal{N}\left(x_0; \mu_p, P_p\right) \). The smoothing distribution, \( p_\theta(x_{0:T} \mid y_{1:T}) \),  for this class of systems takes the form of a multivariate Gaussian, the numerical values of which are given by the closed-form Kalman smoothing equations \citep{Rauch1965}.
		
		Now consider the proposed smoothing approach applied to this linear system. First, a form for \( q_{\beta}(x_{0:T})\) must be selected. We chose to use a multivariate Gaussian, which also implies each pairwise joint distribution is a multivariate Gaussian. 
		\begin{lemma} \label{SE:lem:sufficient flexiby of a gaussian}
			A multivariate Gaussian distribution for \( q_{\beta}(x_{0:T}) \) satisfies Assumption~\ref{SE:ass:suffienct flexibity}.
		\end{lemma}
		\begin{proof}
			See Appendix~\ref{SE:app:proof of sufficient flexiby of a gaussian}.
		\end{proof}
		
		For the linear case being considered, the multivariate Gaussian choice implies that \(  q_{\beta}(x_{0:T}) \) and \( p_\theta(x_{0:T} \mid y_{1:T}) \) have the same functional form, which in turn means that there exists some \(\beta^\star\) such that \(  q_{\beta^\star}(x_{0:T}) \) is identical to \( p_\theta(x_{0:T} \mid y_{1:T}) \). For this \(\beta^\star\) we have that 
		\begin{align}
			\text{KL}[ q_{\beta^\star}(x_{0:T})  \mid\mid p_\theta(x_{0:T} \mid y_{1:T} )] = 0, 
		\end{align}
		and therefore 
		\begin{align}
			\log p_\theta(y_{1:T}) = \mathcal{L}\left(\beta^\star\right).
		\end{align}
		As \( \mathcal{L}\left(\beta\right) \) is a lower bound to \( \log p_\theta(y_{1:T}) \) this implies that \( \beta^\star\) can be found via
		\begin{align}
			\beta^\star &= \arg\max_\beta  \quad \mathcal{L}(\beta),
		\end{align}
		where, due to the system being linear, the integrals within  \(\mathcal{L}(\beta)\) have a tractable closed-form solution and the resulting distribution \(  q_{\beta^\star}(x_{0:T}) \) will match the Kalman smoothed solution. This occurs due to the fact that \(  q_{\beta}(x_{0:T}) \) is sufficiently flexible to exactly match \( p_\theta(x_{0:T} \mid y_{1:T} ) \) and the maximisation has indeed driven any difference between these distributions to zero.
		
		We now progress to applying this smoothing approach to general nonlinear state-space models in the form of \eqref{SE:eq:general nonlinear model} with the continued assumption that \(q_\beta\left(x_{0:T}\right)\), and hence each \( q_\beta(x_{k-1},x_{k}) \), is a multivariate Gaussian distribution. Compared with the application to a linear system, this has two important consequences. 
		
		Firstly, as true for all assumed density approaches to state estimation, \(  q_{\beta}(x_{0:T}) \) will, generally, not be able to exactly match \( p_\theta(x_{0:T} \mid y_{1:T} ) \). As such, the KL divergence term will remain strictly positive, and \( \mathcal{L}\left(\beta\right) \) will not equal \( \log p_\theta(y_{1:T}) \) in general. A related consequence is that when the true state distribution is far from Gaussian, such as when multi-modal, the method developed in this paper cannot necessarily be expected to perform well, a limitation shared with all assumed Gaussian approaches to state estimation.
		
		The second, more influential, consequence is that \(\mathcal{L}(\beta)\) becomes intractable to compute due to the general nonlinear integrals within \(I_2\left(\beta\right)\) and \(I_3\left(\beta\right)\). Since evaluation of these integrals is necessary they will be numerically approximated, the approximations to \(I_2\left(\beta\right)\),  \(I_3\left(\beta\right)\), and hence \(\mathcal{L}(\beta)\) are given by \(\hat{I}_{2}\left(\beta\right)\), \(\hat{I}_{3}\left(\beta\right)\), and \(\hat{\mathcal{L}}(\beta)\) respectively. As the optimisation problem of \eqref{SE:eq:highest_level_optim} is intractable it is replaced by
		\begin{subequations} \label{SE:eq:high_level_approx_optim}
			\begin{align} 
				\beta^\star &= \arg\max_\beta  \quad\hat{\mathcal{L}}(\beta) \\
				&= \arg\max_\beta \quad I_1\left(\beta\right) + \hat{I}_{2}\left(\beta\right) + \hat{I}_{3}\left(\beta\right) - I_4\left(\beta\right).
			\end{align}
		\end{subequations}
		
		This approach to smoothing requires only one minor adaptation when applied to the more general class of nonlinear state-space models given by \eqref{SE:eq:general nonlinear model} as opposed to linear systems. This single difference is that an integral must be numerically approximated rather than calculated in closed form.
		
		This approach to assumed Gaussian state estimation possesses some important differences from more typical methods, such as the EKS/URTSS. One difference is that this approach is iterative (via the optimisation), which allows all function approximations to be about the approximated posterior distribution and improves performance compared to noniterative methods \citep{GarciaFernandez2015,Barfoot2017}. Another difference is that this approach is not based on Kalman-like corrections, which, for nonlinear systems, is limiting \citep{Darling2017}.

	\subsubsection{Tractable Gaussian Assumed Density Smoother}
		In this section, a parametrisation of the assumed state distribution is detailed. The given parametrisation is selected as it allows \eqref{SE:eq:high_level_approx_optim} to become a tractable optimisation problem with readily obtainable first- and second-order derivatives.  The resulting  optimisation problem is in a form that can be directly handled by standard nonlinear programming routines.
		
		Each joint normally distributed assumed density will be represented by
		\begin{align}
			q_{\beta_{k}}\left(x_{k-1},x_{k}\right) = \mathcal{N}\left(	\begin{bmatrix}
				x_{k-1} \\
				x_{k}
			\end{bmatrix} ; 
			\begin{bmatrix}
				\mu_k \\ 
				\bar{\mu}_k
			\end{bmatrix},
			P_k^{\frac{\Transp}{2}} P_k^{\frac{1}{2}}
			\right),
		\end{align}
		where \( \mu_k \in \mathcal{R}^{n_x \times 1}\), \( \bar{\mu}_k \in \mathcal{R}^{n_x \times 1}\), and \(P_k^{\frac{1}{2}} \in \mathcal{R}^{2n_x \times 2n_x}\) describe the joint state distribution and
		\begin{align}
			P_k^{\frac{1}{2}} = \begin{bmatrix}
				A_k & B_k \\
				0 & C_k
			\end{bmatrix},
		\end{align}
		where \( A_k, B_k, C_k \in \mathcal{R}^{n_x \times n_x} \) and \( A_k, C_k \) are upper triangular. The  parameters in \(\beta_{k}\) are now given as 
		\begin{align}
			\beta_{k} =  \{ \mu_k, \bar{\mu}_k, A_k, B_k, C_k \},
		\end{align}
		and they are related to \(\beta\) by
		\begin{align}
			\beta = \{\beta_1, \beta_2, \dots, \beta_T \}.
		\end{align}
		Furthermore, we have to ensure that the marginal distribution \(q_\beta(x_k)\) is the same independently of if it is calculated from \( q_{\beta}\left(x_{k-1},x_{k}\right) \) or \(q_{\beta}\left(x_k,x_{k+1}\right) \). This is expressed by the nonlinear constraints defined by the feasible set
		\begin{align}
			\Omega := \{\beta \in \mathcal{R}^{n_\beta} \mid c_k\left(\beta\right) = 0, \quad k = 1, \dots, T-1 \}, 
		\end{align}
		where \(n_\beta = T\left(2n_x + n_x\left(2n_x+1\right)\right)\) and
		\begin{align}
			c_k\left(\beta\right) = \begin{bmatrix}
				\mu_{k+1} - \bar{\mu}_k \\
				\text{vech}(B_{k}^\Transp {B}_{k} + {C}_{k}^\Transp {C}_{k}) - \text{vech}({A}_{k+1}^\Transp {A}_{k+1})
			\end{bmatrix},
		\end{align}
		where \( \text{vech}(X) \), defined in \citep{Henderson1979}, stacks the distinct elements of the symmetric matrix \(X\).
		
		The important property of this parametrisation is that the Cholesky decomposition, \(P_k^{\frac{1}{2}} \), of each joint distribution is directly available from \(\beta_k \). This property allows the components of \(\hat{I}_{2}\left(\beta\right)\) and \(\hat{I}_{3}\left(\beta\right)\) corresponding to each time step, denoted \(\hat{I}_{2_k}\left(\beta\right)\) and \(\hat{I}_{3_k}\left(\beta\right)\), respectively, to be easily approximated using Gaussian quadrature as
		\begin{subequations} \label{SE:eq:approx_integral}
			\begin{align} 
				\hat{I}_{2_k}\left(\beta\right) &= \sum_{j = 1}^{n_s} w_j  \log p_\theta(\bar{x}_{k}^j \mid x_{k-1}^j), \\
				\hat{I}_{3_k}\left(\beta\right) &= \sum_{j = 1}^{n_s} w_j  \log p_\theta( y_{k} \mid \bar{x}_{k}^j).
			\end{align}
		\end{subequations}
		The weights $w_j$  and the $n_s$ sigma points are denoted by \( \varepsilon^j_k  \in \mathcal{R}^{2n_x \times 1}\), where
		\begin{align}
			\varepsilon^j_k = \begin{bmatrix}
				x_{k-1}^j         \\
				\bar{x}_{k}^j
			\end{bmatrix},
		\end{align}
		and are given by linear combinations of the joint mean
		\(
		\begin{bmatrix}
			\mu_k^\Transp, 
			\bar{\mu}_k^\Transp
		\end{bmatrix}^\Transp
		\)
		and the columns of \(P_k^{\frac{\Transp}{2}}\) according to the specific Gaussian quadrature scheme utilised. Specific quadrature schemes are discussed further in Section~\ref{SE:sec:integration}. The sigma points, \(\varepsilon^j_k\), being linear combinations of the elements of \(\beta_k\) is important as it significantly simplifies the calculation of both first- and second-order derivatives used in the optimisation. Using \eqref{SE:eq:approx_integral}, the integrals \(I_2\left(\beta\right)\) and \(I_3\left(\beta\right)\) are approximated as
		\begin{subequations}
			\begin{align}
				\hat{I}_{2}\left(\beta\right) &= \sum_{k=1}^{T} \hat{I}_{2_k}\left(\beta\right), \\
				\hat{I}_{3}\left(\beta\right) &= \sum_{k=1}^{T} \hat{I}_{3_k}\left(\beta\right).
			\end{align}
		\end{subequations}
		Note that, by using Gaussian quadrature to approximate \(I_2(\beta)\) and \(I_3(\beta)\), we have assumed that both \( \log p_\theta(x_{k} \mid x_{k-1})\) and \( \log p_\theta( y_k \mid x_{k})\) can be evaluated pointwise.

		The likelihood lower bound \( \mathcal{L}\left(\beta\right)\) can now be approximated by \( \hat{\mathcal{L}}\left(\beta\right) \) as
		\begin{align}
			\hat{\mathcal{L}}\left(\beta\right) &= I_1\left(\beta\right) + \hat{I}_{2}\left(\beta\right) + \hat{I}_{3}\left(\beta\right) - I_4\left(\beta\right),
		\end{align} 
		where, under the assumption of a Gaussian prior, the full calculations of \(I_1\left(\beta\right)\), \(\hat{I}_{2}\left(\beta\right)\), \(\hat{I}_{2}\left(\beta\right)\), and \( I_4\left(\beta\right)\) are detailed in Appendix~\ref{SE:app:full_log_bound_approx_eqn}.
		
		The proposed approach to assumed density smoothing for nonlinear models in the form of \eqref{SE:eq:general nonlinear model} is now tractably given as performing
		\begin{subequations} \label{SE:eq:tractable_optim}
			\begin{align} 
				\beta^\star = \arg\max_\beta \quad &\hat{\mathcal{L}}(\beta),  \\
				\text{s.t.} \quad &\beta \in \Omega.
			\end{align}
		\end{subequations}
		This constrained optimisation problem is of a standard form and can be directly solved using standard nonlinear programming routines. This optimisation is robustly performed without introducing any further approximations, to a local maximum, utilising exact first- and second-order derivatives.
		
		The assumed density approximation to each pairwise smoothed joint distribution are given by \(q_{\beta^\star_{k}}(x_{k-1},x_{k})\) for \(k \in 1, \dots, T\) where
		\begin{align}
			\{\beta_1^\star, \beta_2^\star, \dots, \beta_T^\star \} = \beta^\star.
		\end{align}
		The resulting smoother is summarised in Algorithm~\ref{SE:alg:smoothing}.
		\begin{figure}[!t] 
			\begin{algorithm}[H] 
				\caption{Smoothing}
				\label{SE:alg:smoothing}
				\begin{algorithmic}	[1]
					\renewcommand{\algorithmicrequire}{\textbf{Input:}}
					\renewcommand{\algorithmicensure}{\textbf{Output:}}
					\REQUIRE Measurements \(y_{1:T}\), prior mean \(\mu_p\) and covariance \(P_p\), initial estimate of \(\beta\). 
					\ENSURE Smoothed distributions \(q_{\beta^\star_{k}}(x_{k-1},x_{k})\) for \(k \in 1, \dots, T\).
					\STATE Obtain $\beta^\star$ from \eqref{SE:eq:tractable_optim} initialised at \(\beta\).
					\STATE \( \{\beta_1^\star, \beta_2^\star, \dots, \beta_T^\star \} = \beta^\star \).
				\end{algorithmic}
			\end{algorithm}
		\end{figure}
	
		The constrained formulation of the optimisation problem differs from the approach usually taken in VI, where unconstrained formulations are typically used. Here the optimisation problem has been over-parametrised, and nonlinear constraints introduced to maintain equivalence to the original optimisation problem. This decouples each joint distribution within the calculation of \(\hat{\mathcal{L}}(\beta)\); without introducing assumptions. The benefit of this decoupling is a significant simplification in the calculation of derivatives, particularly the second-order derivatives. This approach differs from typical VI approaches, which reduce the complexity of the optimisation in differing ways. Indeed, one of the most common approaches in VI \citep{Blei2017} is to assume a \emph{mean-field variational family} \citep{Bishop2006}, or decoupled, assumption in the form of \( q_\beta(x_{0:T}) \) to simplify calculations. In the context of state-space models, this corresponds to assuming mutual independence between consecutive time steps or even individual states, an inappropriate assumption for state-space models. In contrast, the proposed constrained approach allows the optimisation problem to be simplified without introducing avoidable and undesirable assumptions.

	% !TeX root = VI_state_estimation_tsp.tex

\subsection{Filtering} \label{SE:sec:filtering}
	In this section, one step of the recursive assumed density filtering problem is considered. That is, given an assumed Gaussian density for the previous time step, \(p(x_0) = \mathcal{N}\left(x_0;\mu_p,P_p\right)\), and the measurement for the current time step, \(y_1\), compute a Gaussian approximation of \(p(x_1 \mid y_1)\).
	
	As this is a special case of the smoothing problem with \(T = 1\), the proposed approach to filtering consists of two steps. Firstly, perform smoothing with a single measurement. Secondly, marginalise the resultant joint distribution to obtain a Gaussian approximation of \(p(x_1 \mid y_1)\). Algorithm~\ref{SE:alg:filter} provides the proposed approach to filtering for a single time step and results in unconstrained optimisation problems. For a sequence of measurements, \(y_{1:T}\), \(T\) applications of Algorithm~\ref{SE:alg:filter}, performed sequentially for each time step, is required and hence has a computational complexity linear with the number of time steps considered, i.e. \(\mathcal{O}\left(T\right)\). 
	\begin{figure}[!t]
		\begin{algorithm} [H]
			\caption{Filtering for a single time step}
			\label{SE:alg:filter}
			\begin{algorithmic}	[1]	
				\renewcommand{\algorithmicrequire}{\textbf{Input:}}
				\renewcommand{\algorithmicensure}{\textbf{Output:}}
				\REQUIRE Measurement \(y_{1}\), prior mean \(\mu_p\) and covariance \(P_p\), initial estimate of \(\beta\).
				\ENSURE Filtered distribution \(\mathcal{N}\left(x_1;\bar{\mu}_1,B_1^\Transp B_1 + C_1^\Transp C_1 \right) \).
				\STATE Obtain \(\beta^\star\) using Algorithm~\ref{SE:alg:smoothing} with \(T = 1\) initialised at \(\beta\).
				\STATE Extract \(\bar{\mu}_1\), \(B_1\), and \(C_1\) from \(\beta^\star\).
			\end{algorithmic}
		\end{algorithm}
	\end{figure}

	% !TeX root = VI_state_estimation_tsp.tex

\subsection{Illustrative Example and Discussion}\label{SE:sec:discussion}
	In this section, properties of the proposed state estimation approach are discussed and compared with existing alternatives. The proposed approach, applied to a single correction step, serves as an illustrative example to assist this discussion. 
	
	In this illustrative example, we consider one correction step using the following measurement model,
	\begin{align}
		y = 0.01 x |x| + \mn, \qquad \mn \sim \mathcal{N}\left(0,0.1\right),
	\end{align}
	with a prior of \(\mathcal{N}(0,0.4)\) and an observed measurement of \(y = 15\).

	Fig.~\ref{SE:fig:correction_compare} shows the resulting state distribution for the proposed approach, the IPLF with 50 iterations (IPLF-50), and a numerically evaluated ground truth. As shown, the IPLF is overconfident, while the proposed approach provides a good approximation of the true posterior. This difference is because of the IPLF's basis in a Kalman filtering framework; this limits the corrections steps to linear corrections \citep{Darling2017}. The proposed approach has no such limitation. Note that, as discussed previously, the variational approach we have used can produce overconfident estimates. This effect, however, occurs when the form of the assumed density cannot closely match the true posterior, a situation that, as seen in Fig.~\ref{SE:fig:correction_compare}, does not occur in this example.
	\begin{figure}[!t]
		\centering
		\includegraphics{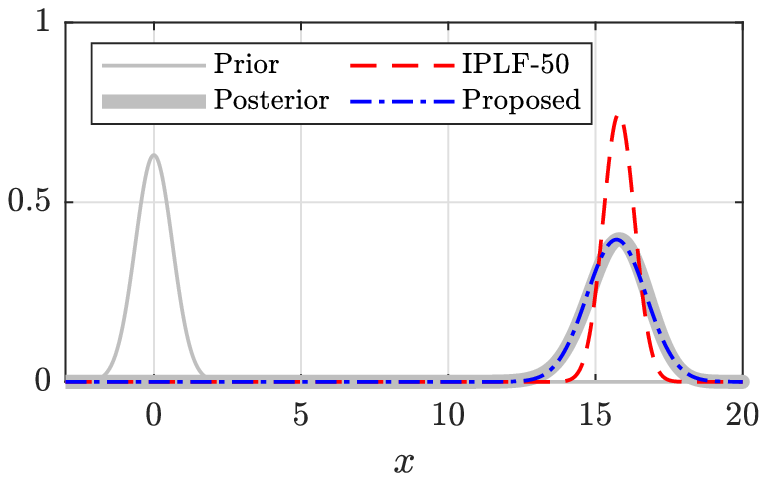}
		\caption{Comparison of state correction approaches illustrating a limitation of using Kalman based corrections for nonlinear measurement models. The Kalman correction based IPLF-50 is overconfident, while the proposed method, which is not limited to linear corrections, closely matches the true posterior.}
		\label{SE:fig:correction_compare}
	\end{figure}

	While Fig.~\ref{SE:fig:correction_compare} only shows the results for a specific prior and observed measurement, many others---using a variety of measurement models---were tested, with the general behaviour remaining consistent with the provided example. During this testing, the IPLF frequently suffered convergence issues. While this issue is acknowledged and addressed in \citep{Raitoharju2018}, it highlights a potential problem that these iterated approaches may suffer. A variety of works exists, see e.g. \citep{Sarkka2020,Skoglund2015,Skoglund2019}, which aim to address convergence issues in differing iterated filters. Generally, these works introduce a variety of safeguards, damping, line-searches, and Levenberg-Marquant based approaches to improve convergence.
	
	Fortunately, the ideas behind these approaches are well established in the extensive quantity of numerical optimisation literature such as \citep{Nocedal2006} and are included in all standard optimisation routines. A result of this is that the proposed state estimation approach requires no further considerations regarding convergence.
	
	This highlights a \textit{significant difference} between the proposed approach and the existing alternatives, namely framing the state estimation problem as an optimisation problem of a standard form. This framing allows a \textit{clean separation} between the aim of the approach, maximising a log-likelihood lower bound, and the specific details of the selected optimisation routine.   
	
	Due to this, contrary to alternatives such as \citep{Raitoharju2018,GarciaFernandez2015,GarciaFernandez2017}, the proposed approach has no tuning parameters. Instead, the optimisation required runs until convergence to a local maximum.

% !TeX root = VI_state_estimation_tsp.tex

\section{Implementation} \label{SE:sec:implmentation}
	In this section, the implementation-specific details of the proposed state estimation approach are presented. Section~\ref{SE:sec:integration} examines the selected form of Gaussian quadrature utilised, Section~\ref{SE:sec:derivatives} address the calculation of derivatives required for efficient optimisation, and Section~\ref{SE:sec:initilisation} deals with how each optimisation problem has been initialised.
	
	\subsection{Integration} \label{SE:sec:integration}
		As commonly performed in Gaussian assumed density state estimation, we have utilised a third-order unscented transform \citep{Julier1997,Wan2000} to perform the Gaussian quadrature approximations required in all sigma-point methods utilised. The parameters \(\alpha\) and \(\kappa\) of the unscented transform were set as \(\alpha = 1\) and \(\kappa = 0\) for the proposed method. 
		
		It is important to highlight that the accuracy of integration can be increased by employing higher-order methods, such as those detailed in  \citep{Kokkala2015,Jia2013,Caflisch1998}. The tradeoff for this increased accuracy is increased computational load and does not require modification of the developed method beyond changing the integration scheme utilised.
	
	\subsection{Derivatives} \label{SE:sec:derivatives}
		To efficiently solve the resulting optimisation problems, exact first- and second-order derivatives have been utilised. Due to the assumed Gaussian distribution, the derivatives of \(I_1\left(\beta\right)\) and \(I_4\left(\beta\right)\) are available in closed-form. The remaining terms \(\hat{I}_2\left(\beta\right)\) and \(\hat{I}_3\left(\beta\right)\), containing general nonlinear functions, are not as straightforward. Fortunately, due to the careful parametrisation of each Gaussian to ensure that the sigma points are linear combinations of the variables that are optimised, exact gradients and exact Hessians can be obtained \textit{without any manual effort}. Standard automatic differentiation tools can be used for this; we have utilised CasADi \citep{Andersson2018}. 
	
		For smoothing, the number of variables and the Hessian dimension continues to grow with the number of measurements. However, the Hessian has a sparse block-diagonal structure. As such, the growth in dimension is not problematic as both the formation and subsequent multiplications, or decompositions, involving the Hessian, possess a linear in time complexity with respect to the number of measurements. Due to this, the calculations required to perform each iteration of the optimisation have a linear in time computational complexity, i.e. \(\mathcal{O}\left(T\right)\).
	
	\subsection{Initialisation} \label{SE:sec:initilisation}
		General nonconvex and nonlinear constrained optimisation problems are central to the proposed state estimation approach. This implies that the initialisations impact both the iterations and runtime required to reach a local minimum and, potentially, the solution obtained; for further details on how optimisation routines function and their properties, such as convergence guarantees, interested readers should refer to literature such as \cite{Nocedal2006} and \cite{Conn2000}. In this section, we examine options for initialising the optimisation problems. 
		
		The optimisation that arises from the filtering procedure in Algorithm~\ref{SE:alg:filter} requires initialisation with a joint Gaussian distribution over the previous and current time step. In the numerical examples provided, a predicted joint distribution calculated using a sigma point method is utilised and found to work well. Differing initialisations may, however, be preferable in some situations. An example of this is when alternative, faster methods, such as a UKF perform well. That solution can then be used as an initialisation for the proposed method and may reduce the number of iterations required to converge.
		
		For the smoothing problem, the initialisation consists of providing an initial estimate to each pairwise joint Gaussian distribution. As optimisation routines do not require the initial point to satisfy the constraints, there is significant flexibility regarding how these can be selected. In some situations, problem-specific knowledge can be used to obtain an initial point. More generally, initialising the smoother using each of the joint distributions produced by a filtering pass has been found to work well. The examples in Section~\ref{SE:sec:examples} use both of these approaches to initialise.
		
		As these are general nonlinear optimisation problems, it is impossible to provide guarantees, or quantitative guidelines, regarding initialisations that ensure avoidance of potentially undesirable or differing local minima. However, during development, both the proposed filtering and smoothing approaches have appeared consistently robust with respect to initialisation, provided the use of an ill-defined `reasonable' initial point.

% !TeX root = VI_state_estimation_tsp.tex

\section{Examples} \label{SE:sec:examples}
	In this section, three numerical examples which profile the proposed approach to state estimation are presented; a challenging scalar system (Section~\ref{SE:sec:nonlinear scaler system}), a problem representative of a simple robotic system (Section~\ref{SE:sec:robot_example}), and a target tracking example using a von Mises-Fisher distribution (Section~\ref{SE:sec:target_tracking_example}).
	
	We have also used the proposed approach to estimate the states of randomly generated multi-dimensional linear systems, with each optimisation randomly initialised. All filtered states, smoothed states, and likelihood calculations matched the results obtained using the Kalman filtering and smoothing equations.
	
	In all examples, the optimisation problems within the proposed approach are run to convergence to a feasible local maximum, subject to an optimality tolerance of \num{1e-10} using \texttt{fmincon} \citep{MATLAB2020}. The exception to this is the smoothing problem in Section~\ref{SE:sec:robot_smoothing}, where an alternative trust region solver with an improved rate of convergence, subject to the same tolerance, is used.

	% !TeX root = VI_state_estimation_tsp.tex

\subsection{Nonlinear Scalar System}\label{SE:sec:nonlinear scaler system}
	In this example, we consider 50 time steps of the following nonlinear system
	\begin{subequations} \label{SE:eq:growth model}
		\begin{align}
			x_{k+1} &= 0.9 x_k + 10 \frac{x_k}{1 + x_k^2} + 8\cos(1.2 k)  + \pn_k,\\
			y_k &= 0.05 x_k^2 + \mn_k,
		\end{align}
	\end{subequations}
	where \( \pn_k \sim \mathcal{N}(0,1)\), \( \mn_k \sim \mathcal{N}(0,1)\) and $x_0\sim \mathcal{N}(5,4)$. This is identical to the setup used in \citep{GarciaFernandez2017}, where the IPLS idea was profiled. It is well-known to be a challenging state estimation problem possessing multi-modal state distributions.
	
	For this system, we will compare the performance of the URTSS, the IPLS, and the proposed smoother. The proposed smoother was initialised with the proposed filter. For the IPLS, one filter and 50 smoothing iterations, denoted IPLS-1-50, were used. As per \citep{GarciaFernandez2017}, using one filter iteration, with many smoothing iterations, performs best on this system.
	
	The performance of these methods is ideally compared using the KL divergence from the true smoothed distribution at each time step, \( \text{KL}\left[ p_\theta(x_{k} \mid y_{1:T} ) \mid \mid q_{\beta^\star}(x_{k}) \right] \). Since \(p_\theta(x_{k} \mid y_{1:T} )\) is intractable, we use a very accurate discrete approximation using a fine grid to implement the recursive Bayesian smoothing equations \citep{Kitagawa1987}. The KL divergence from this grid-based solution is used to evaluate each assumed density approach. The symmetric KL (SKL) divergence and the Jensen-Shannon (JS) divergence \citep{Lin1991} have been calculated similarly.
	
	Fig.~\ref{SE:fig:growth_example_smoothed_KL} shows box plots \citep{Tukey1972} of these divergences at each time step for 100 differing realisations of \eqref{SE:eq:growth model} with an initial state of \(x = 5\). The consistently lower divergences achieved by the proposed smoother shows it has outperformed the alternative methods.
	\begin{figure}[!t]
		\centering
		\includegraphics{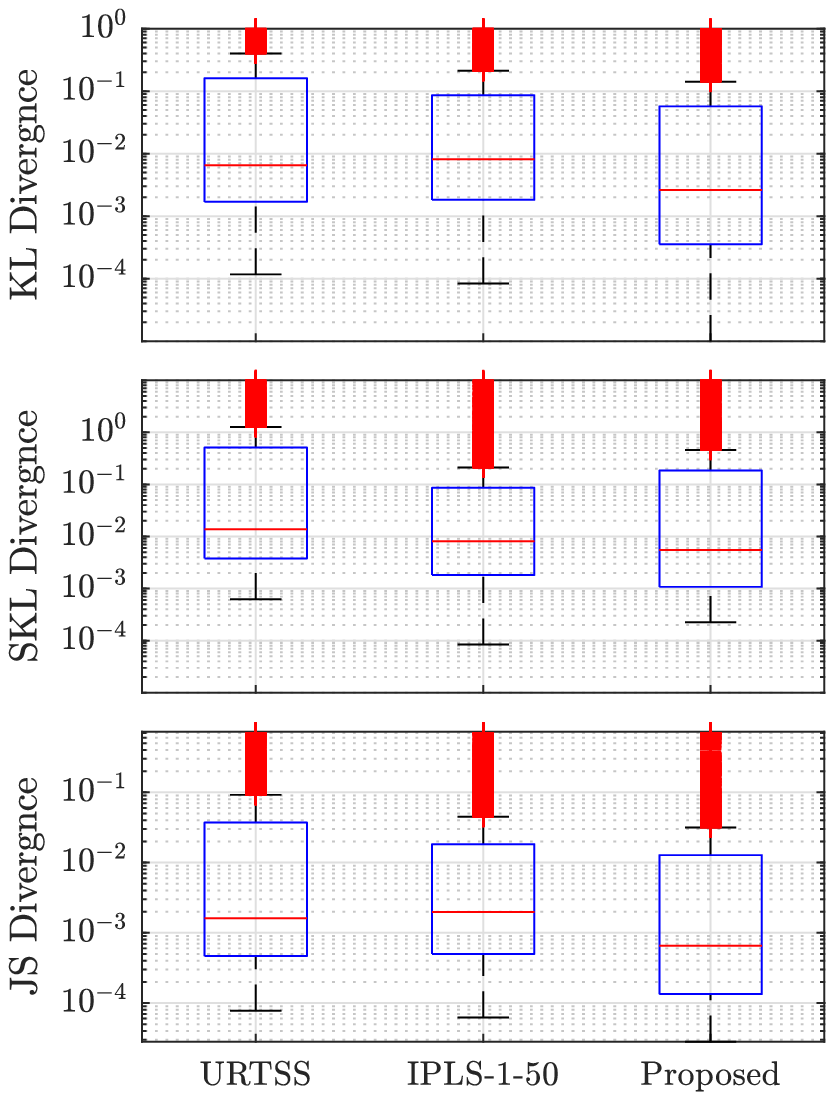}
		\caption[Divergences between the marginal smoothed state distributions and the grid-based ground truth.]{Box plot of the divergences between the marginal smoothed state distributions and the grid-based ground truth for 100 realisations of system \eqref{SE:eq:growth model}, each with 50 time steps.}
		\label{SE:fig:growth_example_smoothed_KL}
	\end{figure}
	
	This improvement has come at the cost of additional runtime, with the median runtimes for the URTSS, the IPLS-1-50, and the proposed smoother being \SI{4.4}{\milli\second}, \SI{257.5}{\milli\second}, and \SI{703.5}{\milli\second}, respectively. The proposed smoother required a median of 15 and a maximum of 53 optimisation iterations to converge.
	
	In this example, we have considered a challenging system that possesses non-Gaussian and multi-modal smoothed state distributions. Due to these properties, no Gaussian assumed density approach is expected to perform well at every time step. This limitation is shown in Fig.~\ref{SE:fig:growth_example_smoothed_dists} where all the assumed Gaussian approaches are seen to be overconfident. Despite this, we have demonstrated that the proposed smoothing method outperforms alternative assumed Gaussian approaches.
	\begin{figure}[!t]
		\centering
		\includegraphics{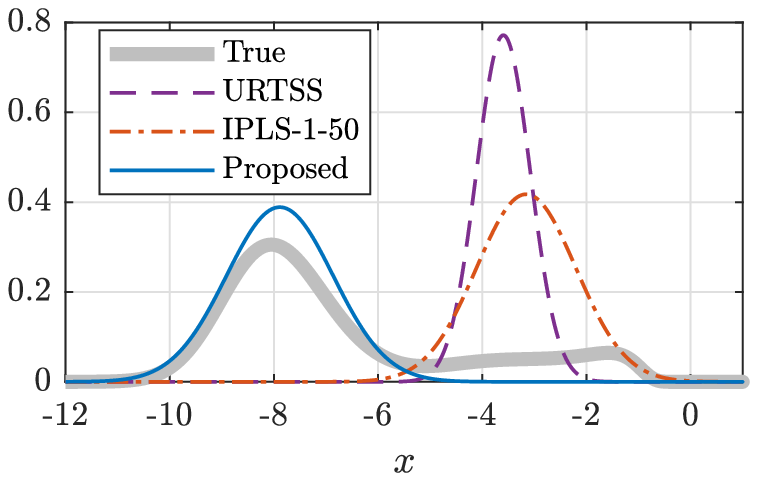}
		\caption[Plot of smoothed distributions for an example time step showing the limitations of assumed Gaussian approaches compared to the true posterior distribution.]{Plot of smoothed distributions for an example time step showing the limitations of assumed Gaussian approaches compared to the true posterior distribution.}
		\label{SE:fig:growth_example_smoothed_dists}
	\end{figure}
	
	% !TeX root = VI_state_estimation_tsp.tex

\subsection{Differential Drive Robot} \label{SE:sec:robot_example}
	In this example, we consider a continuous-time model of a differential drive robot given by
	\begin{align}
		\begingroup
		\renewcommand*{\arraystretch}{1.25}
		\!\!
		\begin{bmatrix}
			\dot{q}_1(t) \\
			\dot{q}_2(t) \\
			\dot{q}_3(t) \\
			\dot{p}_1(t) \\
			\dot{p}_2(t)
			\end{bmatrix} \!\!=\!\! \begin{bmatrix}
			\frac{\cos\left( q_3(t)\right) p_1(t)}{m} \\
			\frac{\sin\left( q_3(t)\right) p_1(t)}{m} \\
			\frac{p_2(t)}{J+ml^2} \\
			\frac{-r_1 p_1(t)}{m} - \frac{ml p_2^2(t)}{\left(J+ml^2\right)^2} + u_1(t) + u_2(t) \\
			\frac{\left(lp_1(t)-r_2\right)p_2(t)}{J+ml^2} + a u_1(t) - a u_2(t)
		\end{bmatrix},
		\endgroup
	\end{align}
	where \(r_1 = 1\), \(r_2 = 1\), \(a = 0.5\), \(m = 5\), \(J = 2\), \(l = 0.15\), \(u_1(t)\) is the force applied to the left wheel, \(u_2(t)\) is the force applied to the right wheel, and the state vector \(x(t) = \left[{q}_1(t), {q}_2(t), {q}_3(t), {p}_1(t), {p}_2(t)\right]^\Transp\) consists of x-position, y-position, heading, linear momentum, and angular momentum states, respectively. 

	A \SI{50}{\second} simulated trajectory is generated using an ODE solver disturbed by noise sampled from \(	\mathcal{N}\left(0,Q\right) \) where \(Q = \text{diag}\left(\num{0.001}, \num{0.001}, \num{1.745e-3}, \num{0.001}, \num{0.001}\right)\) at \SI{0.1}{\second} intervals. Measurements at each interval are obtained according to
	\begin{subequations}
		\begin{align}
			y_k &= \left[{q}_1(t), {q}_2(t), {q}_3(t)\right]^\Transp + \mn_k, \\
			\mn_k &\sim \mathcal{N}\left(0,R\right),
		\end{align} 
	\end{subequations}
	where \( R = \text{diag}\left(0.1^2, 0.1^2, 0.0349^2\right)\).

	The problem of state estimation in the presence of model mismatch is now examined, using an Euler discretisation for each \SI{0.1}{\second} interval. Whilst the Euler discretisation introduces some model mismatch, the mismatch is primarily due to performing the state estimation with the process noise incorrectly given by \(Q = \text{diag}\left(0.01^2, 0.01^2, 0.0035^2, 0.01^2, 0.01^2\right) \) and \(m\), \(J\), and \(l\) randomly given according to
	\begin{align*}
		m \sim \mathcal{U}\left(15,20\right), \quad  J \sim \mathcal{U}\left(0.01,1\right),  \quad l \sim \mathcal{U}\left(0.01,0.5\right).
	\end{align*}

	Due to the model mismatch, this experiment is a sensitivity analysis of the proposed filtering and smoothing methods, compared to alternative methods, with respect to the provided model parameters. For this experiment, the desired outcome for the state estimation approaches is a low sensitivity to the provided parameter estimates, i.e. the state estimation method is not excessively dependent on the provided transition model parameters. This property is important as, for all practical applications, some level of model mismatch is unavoidable and estimates of the model parameters are always used.

	As the measurement model is linear, typical iterated Kalman based filters that separate the prediction and correction steps, such as the IPLF, simplify to a UKF. Contrarily, the proposed approach, which does not have separate prediction and correction steps, remains capable of iterating for this system. While alternative filtering approaches, such as the \(L\)-scan IPLF \citep{GarciaFernandez2017}, which performs smoothing over the last \(L\) measurements, exist and could perform iterations despite the linear measurement model, they are outside the considered scope. Only recursive filters that, given a prior from the previous time step and a single observed measurement, produce a posterior are considered.

	\subsubsection{Filtering} \label{SE:sec:robot_filtering}
		The filtering problem using one set of observed measurements and 500 random sets of mismatched parameters sampled as detailed above are examined. The proposed filter, a fully adapted auxiliary particle filter (FA-APF) \citep{Pitt1999} with \num{5000} particles, and a UKF/IPLF (identical for this system) are considered.
		
		Table~\ref{SE:tab:robot_filter_summary} summarises the results of this experiment and shows the UKF/IPLF--with a square root implementation--was found to be the most sensitive with respect to the provided	model parameters, diverging a high percentage of the time compared to both the FA-APF and the proposed filter. Note that, as the method is deterministic, a diverged result for the UKF/IPLF cannot be remedied by running the method again. Fig.~\ref{SE:fig:robot_filter_heading} shows the estimated heading means of all methods for each successful run of the proposed filter. Despite both methods only diverging once, this figure shows that the state estimates produced by the FA-APF are much less consistent than the proposed method. Fig.~\ref{SE:fig:average_robot_filter_iters} shows that the average number of iterations required for convergence at each time step has remained small for the proposed method.
		\begin{table}[t]
			\caption{Percentage of diverged results and runtimes of the filters.}\label{SE:tab:robot_filter_summary}
			\centering
			\begin{tabular}{lcc}
				\toprule
				Method   & Diverged (\%) & Mean Runtime (s) \\ \midrule
				Proposed &   \num{0.2}   &   \num{3.379}    \\
				FA-APF   &   \num{0.2}   &   \num{1.272}    \\
				UKF/IPLF &   \num{56}    &   \num{0.024}    \\ \bottomrule
			\end{tabular}
		\end{table}
		\begin{figure}[!t]
			\centering
			\includegraphics{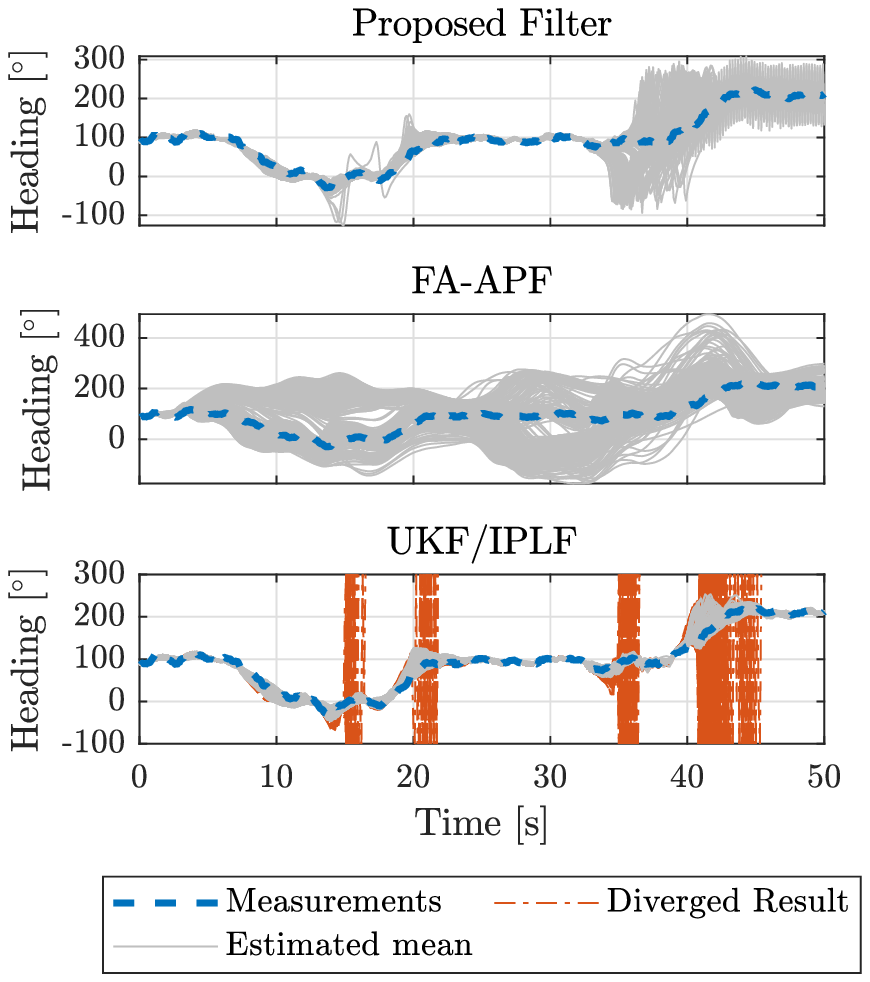}
			\caption[Estimated heading mean for each successful run of the proposed filter and the corresponding results for the FA-APF with \num{5000} particles and the UKF/IPLF.]{Estimated heading mean for each successful run of the proposed filter and the corresponding results for the FA-APF with \num{5000} particles and the UKF/IPLF.}
			\label{SE:fig:robot_filter_heading}
		\end{figure}
		\begin{figure}[!t]
			\centering
			\includegraphics{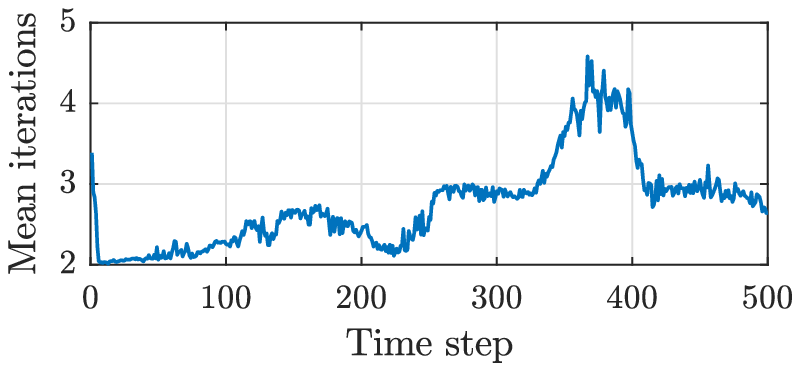}
			\caption[Average iterations for the proposed filtering method to converge.]{Average iterations for the proposed filtering method to converge.}
			\label{SE:fig:average_robot_filter_iters}
		\end{figure}
	
		In this example, we have shown that, at least for this system, the proposed filtering approach is significantly less sensitive than both the UKF/IPLF and a particle filter with respect to the provided model parameters.
		
	\subsubsection{Smoothing}\label{SE:sec:robot_smoothing}
		Smoothing on the first 200 trials considered for the filtering example is now considered using both the proposed smoother and the IPLS with one filter iteration and 50 smoothing iterations, denoted IPLS-1-50. To highlight the flexibility of the proposed smoother the state means are initialised using the observed measurements for the first three states and zeros for momentum states. The joint state covariance for each time step is initially given by \( 0.01^2 I_{10}\). 
		
		The IPLS was unsuccessful 109 times of the 200 attempts, this is due to the first step of the IPLS consisting of a UKF which diverged. Contrarily, the proposed smoother was successful in every run, including the one set of parameters that caused the filter to diverge. Fig.~\ref{SE:fig:smoothed_heading_proposed_200} shows the heading mean of each of the 200 runs. As expected, these results are both closer to the observed measurements and more consistent than the estimates produced when filtering in Fig.~\ref{SE:fig:robot_filter_heading}. The proposed smoother required a median of 69 optimisation iterations, taking a mean of \SI{35.68}{\second}.
		\begin{figure}[!t]
			\centering
			\includegraphics{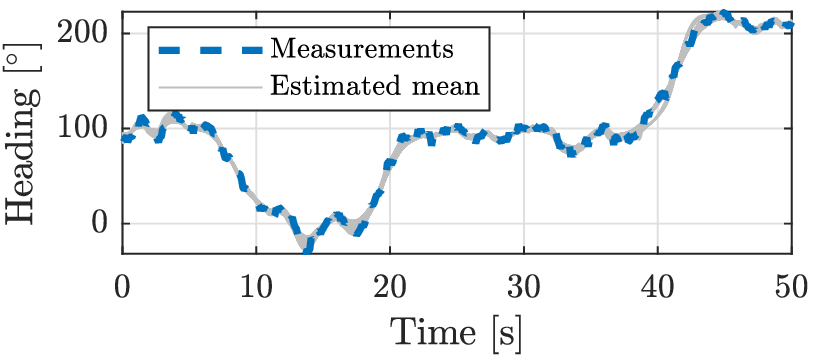}
			\caption[Heading mean for each run of the proposed smoother.]{Heading mean for each of the 200 runs of the proposed smoother.}
			\label{SE:fig:smoothed_heading_proposed_200}
		\end{figure}
		
		The 81 successful runs of the IPLS, however, required an average of \SI{3.513}{\second} and produced similar smoothed states as the proposed method. Due to the linear measurement model, this similarity is unsurprising. The exception is the states for the first few seconds; Fig.~\ref{SE:fig:smoothed_state_initial_compare} shows the smoothed means of each state and illustrates the proposed smoother delivered estimates that are more consistent and reasonable.
		
		The handling of the state prior, \( p(x_0) \), is believed to be the cause of this difference. The IPLS, and indeed all forward-backwards RTS style smoothers, begin with a prediction from the prior for every iteration. This is regardless of the potentially significant difference between the smoothed state estimate for the initial time step and the state prior. While producing the exact smoothing solution for linear systems, there does not appear to be any justification for this initial prediction for approximate assumed density smoothers.
		
		Contrarily, the proposed smoother handles the state prior as an expectation using the estimated smoothed state at the initial time step; at no point is a prediction from the state prior performed. This different handling of the prior is how the proposed filtering approach can iterate over both the process and measurement models compared to just the measurement model. It is also worth noting that, for this example, the mean of the state prior matched the true initial state. In situations where the state prior is far from the smoothed state at the initial time step, this improved handling of the prior would become more pronounced compared with methods that predict from the prior.
		\begin{figure}[!t]
			\centering
			\includegraphics{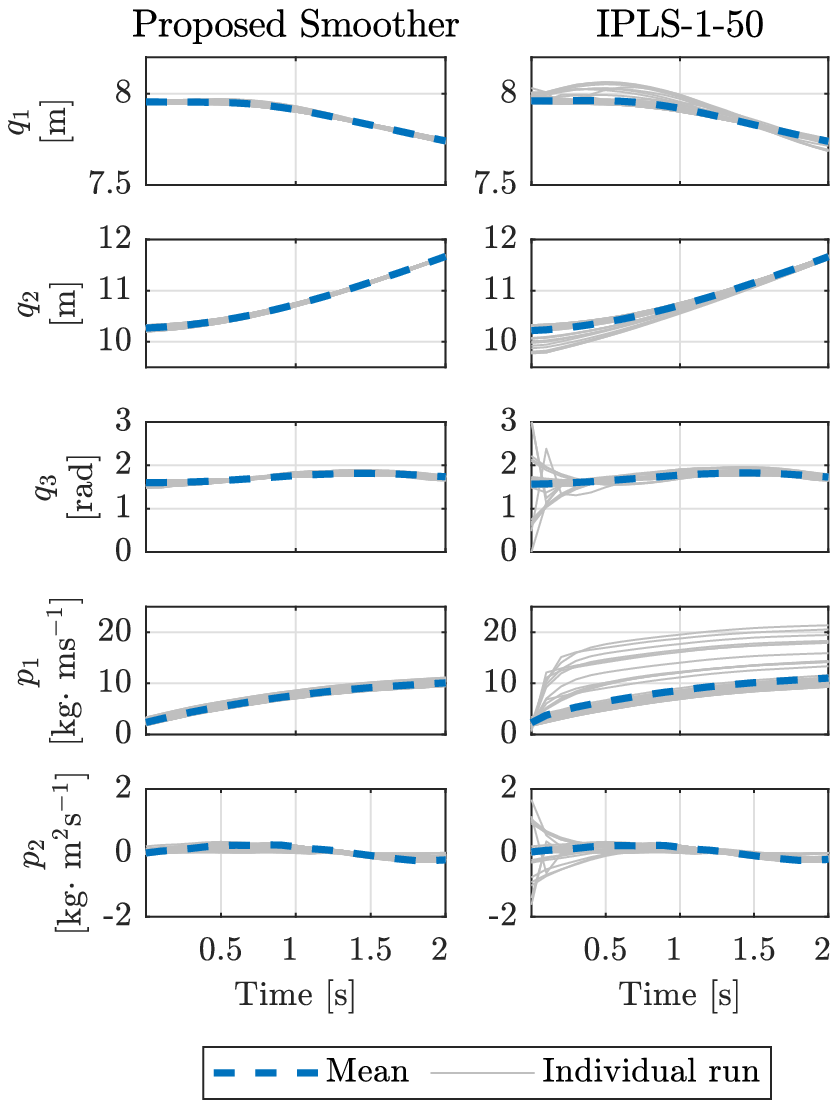}
			\caption[Comparison of the smoothed state means of the initial time steps between the proposed smoother and the IPLS]{Comparison of the smoothed state means of the initial time steps between the proposed smoother (left) and the IPLS-1-50  approach (right), showing that the proposed algorithm is more consistent. Each successful trial is shown in grey and the mean of the successful trials shown in blue.}
			\label{SE:fig:smoothed_state_initial_compare}
		\end{figure}
		
		In this section, the effectiveness of the proposed approach on a multi-dimensional nonlinear system has been demonstrated and the reduced sensitivity with repesct to the provided model parameters shown. While requiring more runtime than alternatives such as the IPLS, the computational cost is not excessive and requires only a moderate amount of iterations to converge. Perhaps more importantly, particularly for the smoothing scenario typically performed offline, is the proposed smoother appears to be more robust and more consistent than alternative methods.
	
	% !TeX root = VI_state_estimation_tsp.tex

\subsection{Target Tracking}\label{SE:sec:target_tracking_example}
	In this example, we consider a 2-D target tracking problem with a four-dimensional state vector, \( x = [p_x, v_x, p_y, v_y]^\Transp\), consisting of \(x\) and \(y\) position and velocity states. This system is observed using three bearing-only sensors, each modelled using a von Mises-Fisher (VMF) distribution. The process model for this system is given by
	\begin{subequations}
		\begin{align}
			x_{k+1} &= \left(I_2 \otimes \begin{bmatrix}
				1 & \tau \\	
				0 & 1
			\end{bmatrix}\right) x_k + \pn_k, \\
			\pn_k &\sim \mathcal{N}\left(0,q I_2 \otimes \begin{bmatrix}
				\frac{\tau^3}{3} & \frac{\tau^2}{2} \\
				\frac{\tau^2}{2}  & \tau
			\end{bmatrix} \right),
		\end{align}
	\end{subequations}
	where \(q = 0.25\), \(\tau = 0.5\), and the initial state prior is given by
	\begin{align}
		p(x_0) &= \mathcal{N}\left(x_0; \mu_p,\Sigma_p\right),
	\end{align}
	where \( \mu_p = [ -100, 7, 0, 5 ]^\Transp\) and \(\Sigma_p = \text{diag}\left(20^2, 1, 1, 1\right)\). The true initial state for each simulation is sampled from this prior. The measurement model for the \(i^{\text{th}}\) sensor is given by
	\begin{subequations}
		\begin{align}
			y^i &\sim \mathcal{V}\left(h^i(x),\kappa\right), \\
			h^i(x) &= \frac{\left[p_x - s_x^i, p_y - s_y^i\right]^\Transp}{ || \left[p_x - s_x^i, p_y - s_y^i\right] ||},
		\end{align}
	\end{subequations}
	where \(\mathcal{V}\left(h^i(x),\kappa\right)\) is a VMF distribution with a concentration parameter of \(\kappa = 200\) and \(s_x^i\), \(s_y^i\) are the \(x\) and \(y\) positions of the \(i^{\text{th}}\) sensor, respectively. The three sensors are located at \(\left[100, 0\right]^\Transp\), \(\left[0, -100\right]^\Transp\), and \(\left[0, 150\right]^\Transp\). Throughout this section, we consider 100 time steps of this system, an example path with the sensor locations marked is shown in Fig.~\ref{SE:fig:vmf_target_tracking_filter}.
	\begin{figure}[!t]
		\centering
		\includegraphics{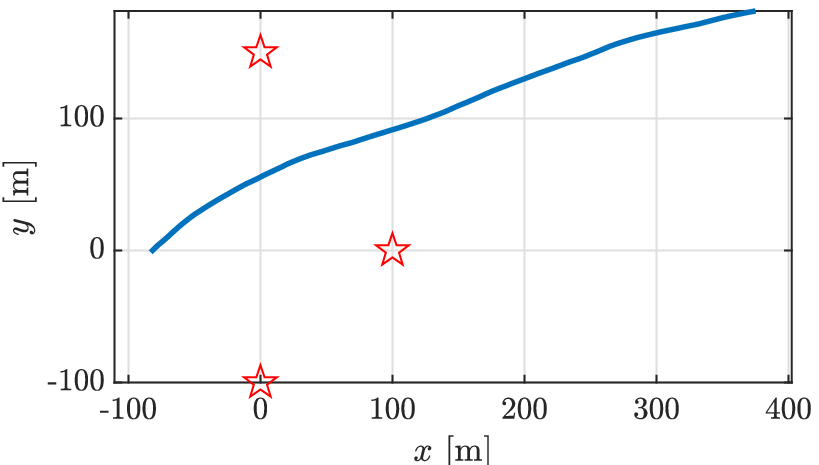}
		\caption[Example path used in the target tracking problem.]{Example path (blue line) used in the target tracking problem with the sensor locations marked by red stars.}
		\label{SE:fig:vmf_target_tracking_filter}
	\end{figure}
	
	This target tracking problem is taken from \citep{GarciaFernandez2019}, where, utilising explicit knowledge of moments of the VMF distributions, the IPLF has been extended to produce the VMF-IPLF which has outperformed a range of alternative Gaussian assumed density filters. We will compare against the sigma point version of the VMF-IPLF with \(i\) iterations which we have denoted as VMF-IPLF-i. 
	
	To assess the performance of the VMF-IPLF, and the proposed filtering approach, a bootstrap particle filter (BPF) \citep{Gordon1993} with many particles is used for comparison. To determine an appropriate quantity of particles, we first perform the BPF 20 times on the same dataset with the number of particles ranging from \num{5e3} to \num{1e6}. Fig.~\ref{SE:fig:vmf_ll_diff} is a box-plot showing the difference between the log-likelihood calculated from the BPF and the approximate lower bound from the proposed filter for each trial. These results highlight two properties; firstly, the log-likelihood calculated using the BPF and \(\hat{\mathcal{L}}\left(\beta^\star\right)\) is close, showing that the lower bound is relatively tight. Secondly, as the particle counts continue to increase, the improvements of the variance in the BPF start becoming minor at the \num{500000} particle mark.
	\begin{figure}[!t]
		\centering
		\includegraphics{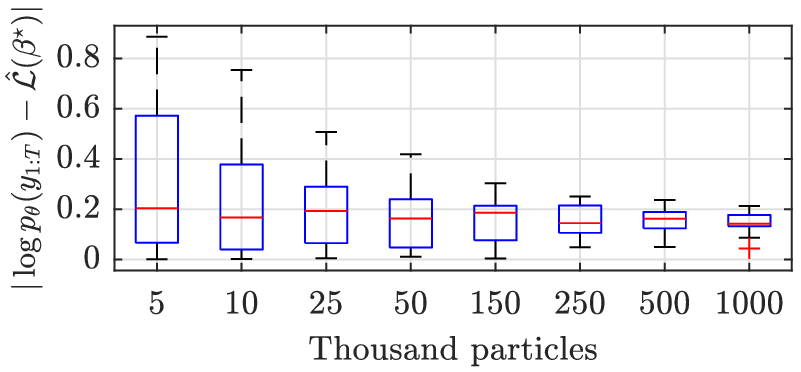}
		\caption[Difference between log-likelihood calculated using a bootstrap particle filter and the proposed filter.]{Difference between log-likelihood calculated using a bootstrap particle filter and the proposed filter on a single dataset for a variety of particle counts each performed 20 times.}
		\label{SE:fig:vmf_ll_diff}
	\end{figure}
	
	Box-plots showing the position error of the filtered mean of the proposed filter and the VMF-IPLF from the BPF mean for each trial are shown in Fig.~\ref{SE:fig:vmf_pos_error} and Fig.~\ref{SE:fig:vmf_pos_error_iplf}, respectively. These results show that the proposed filter is closer to the particle solution with many particles than the VMF-IPLF.
	\begin{figure}[!t]
		\centering
		\includegraphics{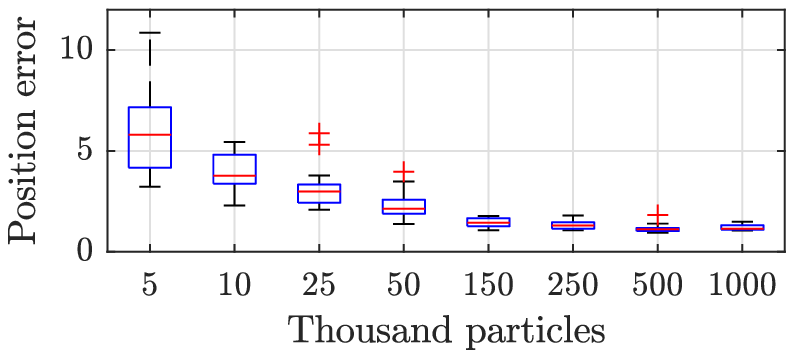}
		\caption[Position error between the state means calculated using a bootstrap particle filter and the proposed filter.]{Position error between the state means calculated using a bootstrap particle filter and the proposed filter on a single dataset for a range of particle counts each performed 20 times.}
		\label{SE:fig:vmf_pos_error}
	\end{figure}
	\begin{figure}[!t]
		\centering
		\includegraphics{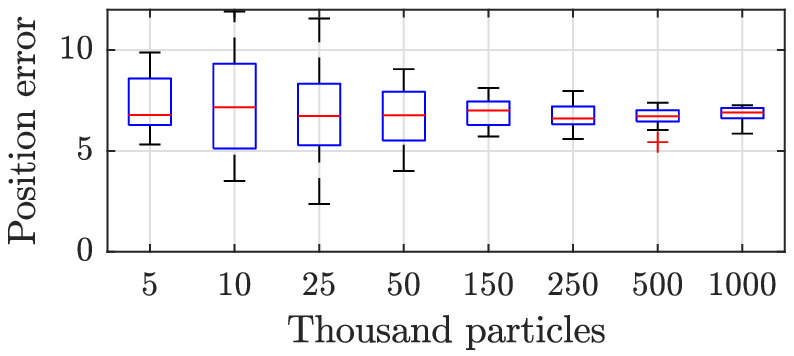}
		\caption[Position error between the state mean calculated using a bootstrap particle filter and the VMF-IPLF.]{Position error between the state means calculated using a bootstrap particle filter and the VMF-IPLF on a single dataset for a range of particle counts each performed 20 times.}
		\label{SE:fig:vmf_pos_error_iplf}
	\end{figure}
	
	Fig.~\ref{SE:fig:vmf_pos_error_500} shows this position error for 500 different state trajectories using \num{500000} particles for the BPF and 50 iterations of the VMF-IPLF. The consistently lower error from the particle solution shows the proposed method has outperformed the VMF-IPLF. As Fig.~\ref{SE:fig:vmf_target_tracking_filter_iters} shows, the average number of iteration required to achieve this is quite low. The median runtimes to calculate the filtered state trajectories over the 100 time steps for the VMF-IPLF-50 and the proposed filter are \SI{705}{\milli\second} and \SI{621}{\milli\second}, respectively; these runtimes are close to the time required for the BPF with \num{10000} particles.
	\begin{figure}[!t]
		\centering
		\includegraphics{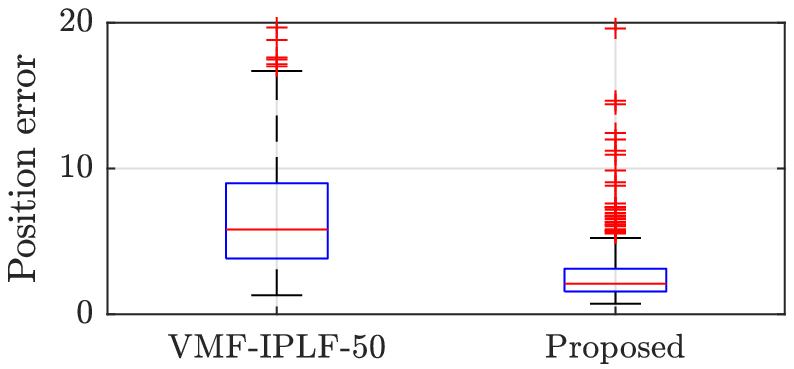}
		\caption[Position error between a BPF solution, the VMF-IPLF, and the proposed approach.]{Position error between a BPF solution using \num{500000} particles, the VMF-IPLF, and the proposed approach over 500 differing state trajectories.}
		\label{SE:fig:vmf_pos_error_500}
	\end{figure}
	\begin{figure}[!t]
		\centering
		\includegraphics{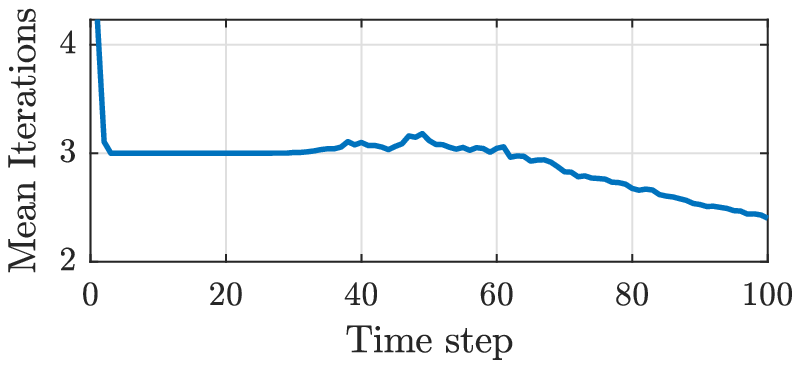}
		\caption[Average iterations required for the proposed filter to converge.]{Average iterations required for the proposed filter to converge for each time step of the 500 differing trajectories.}
		\label{SE:fig:vmf_target_tracking_filter_iters}
	\end{figure}
	
	In this section, the proposed state estimation approach is demonstrated on a target tracking example, modelled with a von Mises-Fisher distribution. This non-Gaussian distribution did not require any alterations to the proposed approach to state estimation beyond modifying the log-likelihood calculation. Despite this, we have outperformed an existing alternative, tailored to a von Mises-Fisher distribution, which itself outperformed a range of assumed density Gaussian filters. This is achieved without significant computational cost, required a minimal number of iterations to converge, and, generally, produced results very close to a particle filter with \num{500000} particles.

% !TeX root = VI_state_estimation_tsp.tex

\section{Conclusion} \label{SE:sec:conculsion}
	In this paper, we have presented an optimisation based approach to the problem of assumed Gaussian state estimation for nonlinear and non-Gaussian state-space models using variational inference in the context of both filtering and smoothing. The resulting optimisation problems, of standard form, are efficiently solved using exact first- and second-order derivatives. Several numerical examples have been provided to demonstrate the performance and robustness of the proposed approach compared to alternative methods. An interesting continuation is to make use of the proposed approximation within a sequential Monte Carlo sampler to also get theoretical guarantees.

\bibliographystyle{IEEEtran}
\bibliography{VI_state_estimation_tsp} % This is the .bib file where the bibliography database is stored

\appendix

% !TeX root = VI_state_estimation_tsp.tex

\section{Proof of Lemma~\ref{SE:lem:optimal_q_factorises}} \label{SE:app:proof of factorisation lemma}
	\begin{proof}	
		The optimal \(\beta\) is given as
		\begin{align} 
			\beta^\star &= \arg\max_\beta \quad \mathcal{L}(\beta),
		\end{align}
		which (recalling the definitions for \(\beta_1\) and \(\beta_2\) given in Assumption~\ref{SE:ass:suffienct flexibity}), can equivalently be written as
		\begin{align}
			\beta_1^\star,\beta_2^\star	&= \arg\max_{\beta_1,\beta_2} \quad g_1\left(\beta_1\right) + g_2\left(\beta_1,\beta_2\right),
		\end{align}
		where
		\begin{subequations}
			\begin{align}
				g_1\left(\beta_1\right) &=  I_1\left(\beta_1\right) + I_2\left(\beta_1\right) + I_3\left(\beta_1\right) \notag \\
				&\quad - \int  q_{\beta_1}(x_{0}) \log  q_{\beta_1}(x_{0})  dx_{0}, \\
				g_2\left(\beta_1,\beta_2\right) &= - \sum_{k=1}^{T} \int  q_\beta(x_{0:k}) \log   q_\beta(x_{k} \mid x_{k-1}, x_{0:k-2}) dx_{0:k}.
			\end{align}
		\end{subequations}
		Therefore, \(\beta_2^\star\) can also be given as
		\begin{align}
			\beta_2^\star	&= \arg\max_{\beta_2} \quad  g_2\left(\beta_1^\star,\beta_2\right).
		\end{align}
		
		However, using Lemma~\ref{SE:lem:conditional diff entropy}, we have
		\begin{align} \label{SE:eq:conditional diff entropy inequality}
			& -\int  q_\beta(x_{0:k}) \log  q_\beta(x_{k} \mid x_{k-1}, x_{0:k-2}) dx_{0:k} \notag \\
			& \quad \le -\int  q_\beta(x_{0:k}) \log   q_\beta(x_{k} \mid x_{k-1}) dx_{0:k}. 
		\end{align}
		Due to Assumption~\ref{SE:ass:suffienct flexibity}, there exists a \(\beta_2\) such that
		\begin{align}
			q_{\beta_1^\star,\beta_2}(x_{k} \mid x_{k-1}, x_{0:k-2}) =  q_{\beta_1^\star,\beta_2}(x_{k} \mid x_{k-1}),
		\end{align}
		For \(\beta_2 = \beta_2^\star\), we therefore have
		\begin{align}
			q_{\beta^\star}(x_{k} \mid x_{k-1}, x_{0:k-2}) =  q_{\beta^\star}(x_{k} \mid x_{k-1}),
		\end{align}
		as, via \eqref{SE:eq:conditional diff entropy inequality}, \emph{any} other selection for \(\beta_2\) is not optimal.
		
		Therefore, \(q_{\beta^\star}\left(x_{0:T}\right)\) is given by
		\begin{align}
			q_{\beta^\star}\left(x_{0:T}\right) = q_{\beta^\star}(x_0) \prod_{k=1}^{T} q_{\beta^\star}(x_{k} \mid x_{k-1})
		\end{align}
	\end{proof}

	\begin{lemma}[Conditional Differential Entropy] \label{SE:lem:conditional diff entropy}
		Given the differential entropy defined by
		\begin{align}
			h(X) &= -\int p(x) \log p(x) dx,
		\end{align}
		and the conditional differential entropy defined by
		\begin{align}
			h(X|Y) &= - \int p(x,y) \log p(x|y) dx dy, 
		\end{align}
		it is known that,
		\begin{align}
			h(X |Y) \le h(X),
		\end{align}
		with equality if and only if $X$ and $Y$ are independent.
	\end{lemma}
	\begin{proof}
		See Theorem 8.6.1 of \cite{Cover2005}.
	\end{proof}

% !TeX root = VI_state_estimation_tsp.tex

\section{Proof of Lemma~\ref{SE:lem:entropy_decomposition}} \label{SE:app:entropy_decomposition}
	\begin{proof}		
		For the purpose of finding each optimal pairwise distribution \(q_{\beta^\star}\left(x_{k-1},x_{k}\right)\), it can be assumed that \(q_\beta\left(x_{0:T}\right)\) factorises according to 
		\begin{align} 
			q_{\beta}\left(x_{0:T}\right) = q_{\beta}(x_0) \prod_{k=1}^{T} q_{\beta}(x_{k} \mid x_{k-1}).
		\end{align}
		Due to Lemma~\ref{SE:lem:optimal_q_factorises}, this assumption merely limits the resultant optimisation problem to search of the optimal subfamily of \(q_\beta\left(x_{0:T}\right)\).
		
		Using this assumption, and directly based on \citep{Vrettas2008} with the notation modified to align with this paper, the entropy term can be decomposed as follows.
		\begin{align*}
			I_4\left(\beta\right) 	&= \int q_\beta(x_{0:T}) \log q_\beta(x_{0:T}) dx_{0:T} \\
									&= \int q_\beta(x_{0:T}) \log  \left( q_\beta(x_0) \prod_{k=1}^{T} q_\beta(x_{k} \mid x_{k-1}) \right) dx_{0:T} \\
								%	&= \int q_\beta(x_{0:T}) \left( \log  q_\beta(x_0)  + \sum_{k=1}^{T} \log q_\beta(x_{k} \mid x_{k-1}) \right) dx_{0:T} \\
									&= \int q_\beta(x_{0:T}) \log  q_\beta(x_0) dx_{0:T} \\
									&\quad + \sum_{k=1}^{T}  \int q_\beta(x_{0:T}) \log q_\beta(x_{k} \mid x_{k-1}) dx_{0:T}  \\
									&= \int q_\beta(x_{0:T}) \log  q_\beta(x_0) dx_{0:T} \\
									&\quad + \sum_{k=1}^{T}  \int q_\beta(x_{0:T}) \log \frac{q_\beta(x_{k-1},x_{k})}{q_\beta(x_{k-1})} dx_{0:T} \\
								%	&= \int q_\beta(x_{0:T}) \log  q_\beta(x_0) dx_{0:T} \notag \\ &\quad+ \sum_{k=1}^{T}  \int q_\beta(x_{0:T}) \left( \log q_\beta(x_{k-1},x_{k}) - \log q_\beta(x_{k-1}) \right) dx_{0:T} \\
									&= \int q_\beta(x_{0:T}) \log  q_\beta(x_0) dx_{0:T} \notag \\ &\quad+ \sum_{k=1}^{T}  \int q_\beta(x_{0:T}) \log q_\beta(x_{k-1},x_{k}) dx_{0:T} \\
									&\quad - \sum_{k=1}^{T}  \int q_\beta(x_{0:T}) \log q_\beta(x_{k-1}) dx_{0:T}  \\
									&= \sum_{k=1}^{T}  \int q_\beta(x_{0:T}) \log q_\beta(x_{k-1},x_{k}) dx_{0:T} \\
									&\quad - \sum_{k=2}^{T}  \int q_\beta(x_{0:T}) \log q_\beta(x_{k-1}) dx_{0:T}  \\
									&= \sum_{k=1}^{T}  \int q_\beta(x_{k-1:k}) \log q_\beta(x_{k-1},x_{k}) dx_{k-1:k} \\
									&\quad - \sum_{k=1}^{T-1} \int q_\beta(x_{k}) \log q_\beta(x_k) dx_{k}.
		\end{align*}
	\end{proof}

% !TeX root = VI_state_estimation_tsp.tex

\section{Proof of Lemma~\ref{SE:lem:sufficient flexiby of a gaussian}} \label{SE:app:proof of sufficient flexiby of a gaussian}
	\begin{proof}
		Without loss of generality, for notational convenience, it is assumed that \(T = 2\). The full assumed density is, therefore, given by 
		\begin{align}
			q_\beta\left(x_{0:T}\right) = 
			\mathcal{N}\left( \begin{bmatrix}
				x_0 \\
				x_1 \\
				x_2
			\end{bmatrix}
			;
			\begin{bmatrix}
				\mu_0 \\
				\mu_1 \\
				\mu_2
			\end{bmatrix},
			\begin{bmatrix}
				\Sigma_{11} & \Sigma_{12} & \Sigma_{13} \\	
				\Sigma_{12}^\Transp & \Sigma_{22} & \Sigma_{23} \\
				\Sigma_{13}^\Transp & \Sigma_{23}^\Transp & \Sigma_{33} 
			\end{bmatrix}
			\right)
		\end{align}
		where \(\beta_1 = \{ \mu_0, \mu_1,\mu_2, \Sigma_{11}, \Sigma_{12}, \Sigma_{22}, \Sigma_{23}, \Sigma_{33} \}\), \(\beta_2 = \Sigma_{13}\), and \(\beta = \{\beta_1,\beta_2\}\).
		
		By inspection, it is clear that \(\beta_2\) does not influence the pairwise marginal distributions. It remains to show that, for a given \(\beta_1\), a \(\beta_2\) exists such that
		\begin{align}
			q_{\beta_1,\beta_2}(x_{2} \mid x_{1}, x_{0}) =  q_{\beta_1,\beta_2}(x_{2} \mid x_{1}).
		\end{align}
		For the Gaussian considered, this occurs when
		\begin{align}
			\Sigma_{33} - \begin{bmatrix}
				\Sigma_{13} \\
				\Sigma_{23}
			\end{bmatrix}^\Transp
			\begin{bmatrix}
				\Sigma_{11} 		& \Sigma_{12}\\	
				\Sigma_{12}^\Transp & \Sigma_{22} \\
			\end{bmatrix} ^{-1} 
			\begin{bmatrix}
				\Sigma_{13} \\
				\Sigma_{23}
			\end{bmatrix} = \Sigma_{33} - \Sigma_{23}^\Transp \Sigma_{22}^{-1} \Sigma_{23}.
		\end{align}
		It, therefore, suffices to show a \(\Sigma_{13}\) can be found such that
		\begin{align}
			\begin{bmatrix}
				\Sigma_{13} \\
				\Sigma_{23}
			\end{bmatrix}^\Transp
			\begin{bmatrix}
				\Sigma_{11} 		& \Sigma_{12}\\	
				\Sigma_{12}^\Transp & \Sigma_{22} \\
			\end{bmatrix} ^{-1} 
			\begin{bmatrix}
				\Sigma_{13} \\
				\Sigma_{23}
			\end{bmatrix} = \Sigma_{23}^\Transp \Sigma_{22}^{-1} \Sigma_{23}.
		\end{align}
		As this occurs for
		\begin{align}
			\Sigma_{13} = \Sigma_{12} \Sigma_{22}^{-1} \Sigma_{23},
		\end{align}
		which always exists as \(\Sigma_{22}\) is a covariance, this both completes the proof and shows how the full Gaussian density can be found using only the pairwise marginals. For \(T>2\), the described process is applied repeatably with appropriate redefinitions of \(\Sigma_{12}\) to suit the increasing problem dimension.
	\end{proof}

% !TeX root = VI_state_estimation_tsp.tex

\section{Calculation of \(\hat{\mathcal{L}}\left(\beta\right)\)} \label{SE:app:full_log_bound_approx_eqn}
	In this section, the full equations to calculate \(\hat{\mathcal{L}}\left(\beta\right)\) are given assuming a Gaussian prior of \(\mathcal{N}\left(x_0; \mu_p, P_p\right)\).
	\begin{subequations}
		\begin{align}
			\hat{\mathcal{L}}\left(\beta\right) &= I_1\left(\beta\right) + \hat{I}_{2}\left(\beta\right) + \hat{I}_{3}\left(\beta\right) - I_4\left(\beta\right),
		\end{align} 
		where
		\begin{align}
			I_1 \left(\beta\right)        & = -\frac{n_x}{2}\log 2\pi - \sum_{i=1}^{n_x} \log P_p^{\frac{1}{2}}(i,i)  - \frac{1}{2}\trace\left(P_p^{-1} A_1^\Transp A_1\right) \notag                     \\
			                              & \quad - \frac{1}{2}\left( P_p^{-\frac{\Transp}{2}}\left(\mu_p - \mu_1\right)\right)^\Transp \left( P_p^{-\frac{\Transp}{2}}\left(\mu_p - \mu_1\right)\right), \\
			\hat{I}_{2}\left(\beta\right) & =  \sum_{k=1}^T \sum_{j = 1}^{n_s} w_j  \log p_\theta(\bar{x}_{k}^j \mid x_{k-1}^j),                                                                                \\
			\hat{I}_{3}\left(\beta\right) & =  \sum_{k=1}^T \sum_{j = 1}^{n_s} w_j  \log p_\theta( y_{k} \mid \bar{x}_{k}^j),                                                                                   \\
			I_4\left(\beta\right)         & = -\frac{\left(T+1\right)n_x}{2}\log 2\pi - \frac{\left(T+1\right)n_x}{2}         	\notag                                                                     \\
			                              & \quad - \sum_{k=1}^{T} \sum_{i=1}^{n_x} \log C_k(i,i) - \sum_{i=1}^{n_x} \log A_1(i,i).
		\end{align}
	\end{subequations}
	Here, the notation \(X(i,i)\) refers to the \(i^{\text{th}}\) diagonal element of~\(X\).	

\end{document}